\documentclass[10pt,journal,compsoc]{IEEEtran}

\ifCLASSOPTIONcompsoc
  \usepackage[nocompress]{cite}
\else
  \usepackage{cite}
\fi

\usepackage{hyperref}
\usepackage{pifont}
\usepackage{booktabs}
\usepackage{amsthm,amssymb,amsthm}
\usepackage{amsmath,amsfonts}
\usepackage{algorithm}
\usepackage{algpseudocode}

\theoremstyle{definition} 
\newtheorem{assumption}{Assumption}

\newcommand{\bH}{\mathbf{H}}
\newcommand{\bth}{\boldsymbol{\theta}}
\newcommand{\ba}{\mathbf{a}}
\newcommand{\br}{\mathbf{r}}

\usepackage{array}
\usepackage[caption=false,font=normalsize,labelfont=sf,textfont=sf]{subfig}
\usepackage{textcomp}
\usepackage{stfloats}
\usepackage{url}
\usepackage{verbatim}
\usepackage{multirow}
\usepackage[table,dvipsnames]{xcolor}
\usepackage{threeparttable}
\usepackage{graphicx}
\usepackage{siunitx}
\usepackage{cite}
\usepackage{tcolorbox}
\tcbuselibrary{skins}
\usepackage{makecell}
\tcbuselibrary{breakable,listings,skins}
\usepackage{enumitem}
\usepackage{cleveref}

\newtheorem{theorem}{Theorem}[section]
\newtheorem{proposition}[theorem]{Proposition}
\newtheorem{lemma}[theorem]{Lemma}
\newtheorem{corollary}[theorem]{Corollary}

\theoremstyle{definition}

\newtheorem{remark}[theorem]{Remark}

\crefformat{equation}{Eq.~(#2#1#3)}
\Crefformat{equation}{Eq.~(#2#1#3)}

\crefrangeformat{equation}{Eqs.~(#2#1#3) to (#4#5#6)}
\Crefrangeformat{equation}{Eqs.~(#2#1#3) to (#4#5#6)}

\crefmultiformat{equation}{Eqs.~(#2#1#3)}{ and (#2#1#3)}{, (#2#1#3)}{ and (#2#1#3)}
\Crefmultiformat{equation}{Eqs.~(#2#1#3)}{ and (#2#1#3)}{, (#2#1#3)}{ and (#2#1#3)}

\theoremstyle{remark}

\usepackage{tikz}
\usetikzlibrary{arrows.meta,positioning,calc}
\usepackage{pgfplots}
\pgfplotsset{compat=1.18}
\newcommand{\herm}{\mathbf{H}}
\newcommand{\thetabold}{\boldsymbol{\theta}}
\newcommand{\achunk}{\mathbf{a}}
\newcommand{\Ksegs}{K}
\newcommand{\horizon}{T}
\newcommand{\actdim}{D}
\newcommand{\actdimc}{D_c}
\newcommand{\Lflow}{\mathcal{L}_{\text{flow}}}
\newcommand{\Lherm}{\mathcal{L}_{\text{herm}}}
\newcommand{\hdh}{\textsc{Hermite-VLA$_{\mathrm{DH}}$}}
\newcommand{\hch}{\textsc{Hermite-VLA$_{\mathrm{CH}}$}}
\newcommand{\hr}{\textsc{Hermite-VLA$_{\mathrm{Reg}}$}}

\def\BibTeX{{\rm B\kern-.05em{\sc i\kern-.025em b}\kern-.08em
    T\kern-.1667em\lower.7ex\hbox{E}\kern-.125emX}}
\usepackage{balance}

\begin{document}

\title{Hermite Curves as Trajectory Priors for Vision-Language-Action Models}
\author{
Qi~Lv, Jianming~Xing, Zhao~Yang, Mingyuan~Yao, Yinan~Shi, Yawei~Jueluo\\Mike~Zheng~Shou, Xiang~Deng
\IEEEcompsocitemizethanks{
	\IEEEcompsocthanksitem Qi Lv, Jianming Xing, and Xiang Deng are with School of Computer Science and Technology, Harbin Institute of Technology (Shenzhen), China, 518055 (email: lvqi@stu.hit.edu.cn, dengxiang@hit.edu.cn).
	\IEEEcompsocthanksitem Qi Lv, Zhao Yang, Mingyuan Yao, Yinan Shi, and Yawei Jueluo are with Jiangsu Cytoderm Intelligent Technology Co., Ltd., China, 215334.
    \IEEEcompsocthanksitem Qi Lv and Mike Zheng Shou are with the Department of Electrical and Computer Engineering and Show Lab, National University of Singapore, Singapore, 117583 (email: mike.zheng.shou@gmail.com).
    \IEEEcompsocthanksitem Mike Zheng Shou and Xiang Deng are corresponding authors.
}}

\markboth{IEEE Transactions on Pattern Analysis and Machine Intelligence}
{Anonymous \MakeLowercase{\textit{et al.}}: Task-Adaptive Visual Chain-of-Thought}

\IEEEtitleabstractindextext{
    \begin{abstract} 

Despite recent progress in Vision-Language-Action (VLA) models for robotic manipulation, the action chunk remains a weakly structured interface. 
Existing work typically flatten each chunk into per-timestep controls, relying on implicit data learning that manifests as jagged motion and boundary discontinuities during physical execution. 
To address these limitations, we introduce Hermite trajectory priors, parameterizing the chunk trajectory as a piecewise cubic Hermite curve defined by endpoint positions and velocities to explicitly enforce smoothness and continuity. 
We instantiate this fixed operator across discrete autoregressive and continuous generative paradigms via three variants: 
(1) Hermite Tokens, which predict quantized boundary variables autoregressively; 
(2) Hermite Scaffold, which decomposes clean actions into a base scaffold and residuals; 
and (3) Hermite Regularization, which applies the prior strictly as an auxiliary training objective.
Across simulation benchmarks and real-robot platforms, Hermite Regularization achieves superior performance among these three variants, improving $\pi_{0.5}$ baseline success rates from 95.9\% to 98.7\% on LIBERO, 85.7\% to 90.9\% on LIBERO-plus, and 63.4\% to 90.0\% across four real-robot tasks without additional inference overhead. 
Trajectory analyses reveal that explicitly structuring trajectory priors serves most effectively as a learning inductive bias rather than a runtime constraint.
Project page is available at \url{https://aopolin-lv.github.io/Hermite/}.

\end{abstract}

}
\maketitle

\IEEEdisplaynontitleabstractindextext
\IEEEpeerreviewmaketitle


\begin{figure*}[t]
    \centering
    \includegraphics[width=\linewidth]{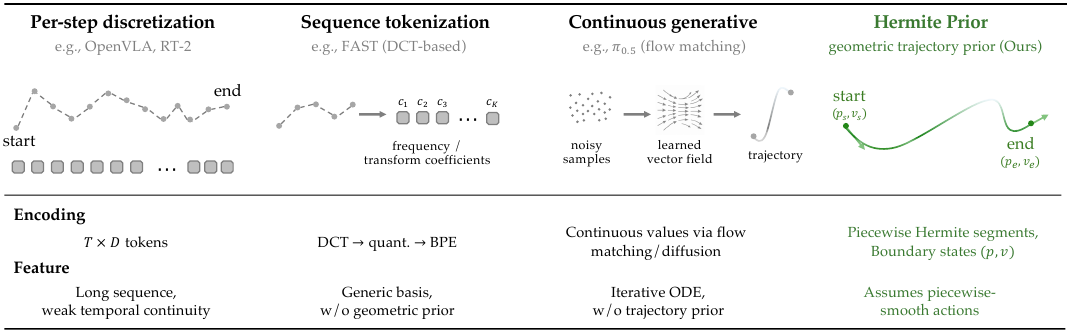}
    \vspace{-3mm}
    \caption{
        Representations of an action chunk. Per-step discretization (e.g., OpenVLA, RT-2) emits one token per timestep, yielding long, weakly structured token sequences. Sequence tokenization (e.g., FAST) compresses the chunk into transform coefficients via a generic (DCT) basis with limited physical prior. Continuous generative decoders (e.g., $\pi_{0.5}$) sample a trajectory from a learned flow-matching field, expressive but geometry-agnostic. Our Hermite trajectory prior represents the chunk by endpoint positions and velocities $(p, v)$, making intra-chunk smoothness and inter-chunk continuity explicit in the action representation rather than leaving them to be learned implicitly from data.
    }
    \vspace{-1mm}
    \label{fig:paradigms}
\end{figure*}

\section{Introduction}
\label{sec:intro}

Vision-Language-Action (VLA) models~\cite{kim2024openvla, black2025pi0, intelligence2025pi05, liu2024rdt} have achieved remarkable success in robotic manipulation by coupling large pre-trained vision-language backbones with action decoders. 
Among various design choices~\cite{zhao2023act, chi2023diffusion}, action chunking, which predicts a short horizon of future actions instead of a single next-step action, has become the prevailing paradigm due to its superior temporal consistency.
Despite its widespread adoption, existing VLAs typically model each action chunk as a flat sequence of per-timestep controls~\cite{kim2024openvla, intelligence2025pi05, li2024cogact, acotvla2025}.
While simple and effective, this representation leaves the action chunk as a weakly structured interface between high-level reasoning and low-level robot control. 
Thus, how to endow the action chunk with the structure of physical motion remains a challenging and insufficiently explored issue.

A key limitation is that a flat per-timestep representation models the chunk as a sequence of discrete control samples, with no notion of the underlying continuous trajectory that connects them. 
This single choice has two consequences. 
(1)~Intra-chunk incoherence: since each timestep within a chunk is parameterized independently, optimization naturally relies on pointwise reconstruction losses rather than trajectory-level constraints~\cite{black2025realtime}. 
Consequently, predictions with similar per-timestep errors receive identical penalties regardless of their smoothness, leaving jagged intra-chunk motions insufficiently penalized during training.
(2)~Inter-chunk discontinuity: Across adjacent chunks, closed-loop replanning concatenates independently decoded action sequences without boundary constraints on position or velocity~\cite{zhou2025beast, yang2026chunkflow}. 
Thus, transitions between consecutive replanning steps receive no continuity enforcement, leaving abrupt inter-chunk jumps unmitigated during execution.
Both failures share the common origin: the flat representation lacks both the temporal smoothness and boundary continuity that physical trajectories inherently require.
We thus view this not as a capability ceiling since large backbones can recover smooth motion given enough data, but as a missing inductive bias whose absence is most damaging under distribution shift and physical execution.
Figure~\ref{fig:paradigms} contrasts existing chunk representations with the proposed prior.

To address these issues, we propose \emph{Hermite trajectory priors} for VLA action modeling. 
The central idea is to expose the continuous trajectory underlying an action chunk, rather than treating the chunk only as a set of independent per-timestep samples.
We represent this trajectory with a piecewise cubic Hermite curve, each segment of which is specified by its boundary conditions, namely the position and velocity at each endpoint.
The corresponding discrete chunk is then obtained by applying a fixed linear basis matrix to these boundary parameters.
This formulation directly targets the two limitations identified above: 
(1)~It explicitly enforces boundary continuity for inter-chunk replanning.
Since the Hermite parameters directly encode endpoint positions and velocities, the boundary conditions governing continuity become an explicit part of the action representation. 
Assigning matched states to adjacent segment knots biases the decoder toward seamless transitions, enabling consecutive chunks to join smoothly during closed-loop execution.
(2)~It structurally guarantees smoothness for intra-chunk motion. 
By construction, each segment is $C^1$-continuous, and the low-order polynomial basis naturally restricts representational capacity to low-frequency components. 
This eliminates high-frequency, jagged predictions that standard pointwise decoders often produce.
In this way, temporal smoothness and boundary continuity, which standard decoders must attempt to learn implicitly from data, are structurally embedded within the action representation.

This prior is also lightweight and  compatible with different action modeling paradigms. 
Reconstructing a dense trajectory from its boundary variables requires only a fixed basis-matrix multiplication, making the operator compact, differentiable, and inexpensive to integrate into existing VLA models.
Using this common operator, we investigate two nested design questions. 
(1)~\textit{Can the same trajectory prior support both discrete and continuous action modeling}?
We answer affirmatively by introducing two specialized realizations: Discrete Modeling with Hermite tokens (\textbf{\hdh{}}), which predicts quantized Hermite boundary parameters autoregressively, and Continuous Modeling with a Hermite Scaffold (\textbf{\hch{}}), which decomposes each action estimate into a structural Hermite scaffold and a per-timestep residual.
(2)~\textit{Within continuous action modeling, must the prior explicitly participate in deployed action prediction}?
To explore this problem, we introduce Continuous Modeling with Hermite Regularization (\textbf{\hr{}}).
This variant employs the Hermite operator strictly as an auxiliary training objective, preserving the exact baseline action head and sampling procedure during inference without introducing runtime latency.

We conduct comprehensive experiments across two simulation benchmarks (LIBERO and LIBERO-plus), along with three real-robot platforms spanning single-arm and bimanual manipulation.
Empirical results demonstrate that Hermite trajectory priors consistently enhance performance across both discrete and continuous VLA paradigms.
Against the shared $\pi$-series baseline, the strongest variant \hr{} improves the average success rate from $95.9\%$ to $98.7\%$ on standard LIBERO, from $85.7\%$ to $90.9\%$ under the distribution shifts of LIBERO-plus, and from $63.4\%$ to $90.0\%$ on four real-robot tasks.
Similarly, the discrete variant \hdh{} improves over its autoregressive baseline $\pi_0$-FAST from $85.5\%$ to $95.4\%$ on LIBERO and from $61.6\%$ to $69.4\%$ on LIBERO-plus. 
Crucially, within continuous action modeling, \hr{} outperforms \hch{} while retaining the baseline action head and sampling procedure at inference.

To summarize, our main contributions are as follows:
\begin{itemize}
    \item Different from existing decoders that treat an action chunk as a flat sequence of per-timestep controls, we introduce a compact cubic Hermite operator that parameterizes action chunks through endpoint positions and velocities, thus embedding explicit intra-chunk smoothness and inter-chunk boundary continuity directly into VLA action learning.
    \item We systematically integrate this trajectory prior into both discrete and continuous VLA action modeling via \hdh{} and \hch{}, while also introducing \hr{}, which enforces trajectory continuity purely through an auxiliary objective to preserve the original inference pipeline.
    \item Beyond empirical success rates, we provide a fine-grained kinematic characterization of the generated motions, demonstrating that the prior effectively lowers jerk, suppresses high-frequency spectral content, and attenuates seam transients on real hardware. Crucially, our findings reveal that trajectory smoothness benefits VLAs primarily as a learning-stage inductive bias rather than an inference-stage constraint, enabling the training-only \hr{} to yield optimal performance at zero runtime overhead.
    
\end{itemize}

\section{Related Work}
\label{sec:related}

\subsection{Vision-Language-Action Models for Manipulation}
\label{sec:related:vla}
Robot manipulation has long been studied through visuomotor imitation learning, where policies map visual observations and robot states to low-level control commands~\cite{levine2016end, wang2025hierdiffusion, finn2016deep, zhang2018deep, zhao2023act, chi2023diffusion,chen2026escapingdiversitytraprobotic}. 
Recent Vision-Language-Action (VLA) models~\cite{groot2025n1, geminirobotics2025, cheang2024gr2, shukor2025smolvla, bu2025univla,zhang2025dreamvlavisionlanguageactionmodeldreamed,wu2026pragmaticvlafoundationmodel,xiaomiroboticsteam2026xiaomirobotics1scalingvisionlanguageactionmodels,luo2026beingh05scalinghumancentricrobot,yuan2026qwenrobotmaniptechnicalreportalignment,sun2026vlajepaenhancingvisionlanguageactionmodel,kong2026affordvlainjectingaffordancerepresentations,lv2024robomp2roboticmultimodalperceptionplanning,zhang2026corticalpolicydualstreamview,li2026globalpriormeetslocal,hu2026abstractioninstantiationlearningbehavioral} extend this paradigm by adapting large VLMs to robot control, enabling policies to condition on visual observations and language instructions while benefiting from large-scale multimodal pretraining~\cite{brohan2022rt1, brohan2023rt2, octo2024octo, kim2024openvla, black2025pi0, intelligence2025pi05, liu2024rdt,li2026aceego0unifyingegocentrichuman}. 
This line of work has been further advanced by larger robot datasets~\cite{khazatsky2024droid, walke2023bridgedata, padalkar2023openx}, open-source training recipes~\cite{kim2024openvla, octo2024octo}, and reasoning-oriented architectures~\cite{li2024cogact, internvlam1, internvlaa1, f12025}.

A central design choice in these systems is the action representation. 
Autoregressive VLA models discretize actions into tokens and emit them with a language-model head~\cite{brohan2023rt2, kim2024openvla, li2024cogact,yang2026actionmaprobotpolicylearning}. 
Continuous generative decoders model the action chunk as a conditional sample, using denoising diffusion~\cite{chi2023diffusion, reuss2023goalconditioned, octo2024octo,11093774} or flow-matching objectives~\cite{black2025pi0, intelligence2025pi05}. 
These directions mainly differ in the functional form of the decoder and the scale of the pretraining data. 
Despite these differences, the models share a common assumption: the chunk is a flat matrix of per-timestep controls. 
They differ in \emph{how} this action chunk is predicted, not in \emph{what} representation of the action chunk. 
This paper instead revisits the representation itself, modeling the chunk as a continuous trajectory rather than a set of independent samples.

\subsection{Action Chunk Representation}
\label{sec:related:chunk}
Action chunking predicts a short horizon of future controls rather than a single next action, and has become a widely adopted mechanism for reducing compounding error and improving temporal consistency in imitation learning~\cite{zhao2023act, chi2023diffusion, reuss2023goalconditioned, kang2026x}. 
The standard representation treats a chunk as a matrix of raw continuous actions, such as end-effector poses, joint targets, or gripper commands~\cite{zhao2023act, chi2023diffusion, octo2024octo}. 
Token-based policies instead quantize actions into per-dimension bins, making robot control compatible with sequence modeling and language-model decoding~\cite{brohan2023rt2, kim2024openvla, li2024cogact, liu2026oatorderedactiontokenization}. 
Other work learns latent action codebooks~\cite{shafiullah2022bet, shafiullah2023vqbet, li2026starlearningdiverserobot} to better capture multimodal action distributions.

Recent work has also explored whole-chunk representations, as surveyed from an action-tokenization perspective in~\cite{zhong2025vlasurvey,zhou2025beast}. 
Most notably, FAST~\cite{pertsch2025fast} applies a discrete cosine transform to the chunk and tokenizes the resulting spectral coefficients, reducing token length and improving the efficiency of autoregressive action decoding. 
These methods show that the representation of the action chunk is itself an important design choice, beyond the backbone and decoding objective. 
Hermite parameterization shares this representation-level view but differs in where the structure lives: spectral coefficients are global descriptors of the whole horizon and offer no handle on the states at chunk boundaries, whereas Hermite coordinates are the local boundary positions and velocities, aligning the representation with the replanning interface at which continuity is required to hold.

\subsection{Trajectory Priors and Motion Primitives}
\label{sec:related:traj}
More recently, classical motion parameterizations~\cite{ijspeert2002movement, paraschos2013probabilistic, bahl2020neural} have been revisited within learning-based policies. 
For instance, movement primitives have been integrated with diffusion models to generate smooth manipulation trajectories~\cite{mpd2024,yang2026abpolicyasynchronousbsplineflow}. 
Concurrently with our work, Spline Policy~\cite{tian2026spline} replaces standard action chunks with spline parameters across several backbones, directly using trajectory parameterizations as the policy output during both training and deployment.
A separate line of work leaves the flat chunk intact and instead attempts to repair inter-chunk seams under asynchronous execution. 
These approaches condition newly predicted chunks on previously generated ones or correct already decoded outputs using external modules~\cite{black2025realtime, black2025trainingtime, vlarail2025, a2c2_2025}, instead of explicitly modeling boundary states within the primary decoder.
Moreover, beyond robotic manipulation, explicit temporal trajectory structures have proven similarly valuable in adjacent forecasting tasks, such as hand motion prediction~\cite{ma2025madiff, ma2026unihand}, object pose tracking~\cite{10949865,11314785,10949865} and temporal action localization~\cite{li2024detal, wang2024talsurvey}.

In contrast to prior work, we systematically investigate the architectural level at which trajectory structure should be integrated into pretrained VLAs. 
We explore Hermite curves both as explicit policy output representations and as implicit trajectory priors. 
Crucially, our empirical analysis reveals that applying the Hermite structure strictly as an auxiliary training objective yields the highest performance. 
This implicit formulation achieves superior efficacy without incurring any inference-time overhead, unlocking a zero-cost deployment regime inaccessible to traditional primitive-as-policy architectures.

\section{Method}
\label{sec:method}

We first describe the problem formulation and revisit the two typical modeling paradigms our variants build on (\Cref{sec:prelim}), then define the Hermite trajectory operator (\Cref{sec:method:operator}), and finally instantiate it at three instantiations of VLA models (\Cref{sec:method:injections}).

\subsection{Preliminary}
\label{sec:prelim}

\subsubsection{Problem Formulation}
\label{sec:setup}

We cast robot manipulation as imitation learning from demonstrations. 
At each decision step the VLA model receives an observation $o=(\mathcal{I},\ell,s)$ consisting of visual images $\mathcal{I}$, a textual instruction $\ell$, and proprioceptive state $s$ (e.g., joint positions). 
Following the common recipe, a pretrained vision-language backbone encodes the observation into a context representation $c$, and an action decoder conditioned on $c$ predicts the robot action $a$.

\vspace{1mm}
\noindent\textbf{Action chunking.}
To maintain the temporal coherence, modern VLAs adopt \emph{action chunk} as the output, instead of the single step action.
At each decision point the decoder predicts a chunk of $\horizon$ consecutive actions,
\begin{equation}
  \achunk \in \mathbb{R}^{\horizon \times \actdim},
  \label{eq:chunk}
\end{equation}
where $\achunk$ is a sequence of $\horizon$ per-step controls, each of dimension $\actdim$, typically comprising an end-effector pose together with a gripper command.

\vspace{1mm}
\noindent\textbf{Closed-loop replanning.}
During deployment, the policy operates in a receding-horizon fashion by executing a predicted chunk, acquiring updated observations, and predicting a fresh chunk.
Consecutive chunks are concatenated at a \emph{replanning seam} every $W$ steps ($W\leq T$), where two independently decoded action sequences meet without explicit position or velocity continuity constraints. 
Although the decoder is trained via behavior cloning to reproduce target chunks, standard training objectives provide no guarantee of smoothness across replanning boundaries, resulting in abrupt velocity jumps, contact disruption, and task failure.

\subsubsection{Discrete Autoregressive Modeling}
\label{sec:ar}

Token-based VLAs~\cite{kim2024openvla,brohan2023rt2,pertsch2025fast} cast action prediction as language modeling. 
A fixed action tokenizer maps the chunk $a$ to a sequence of discrete tokens $q_{1:L}$, and a language-model head factorizes their joint distribution autoregressively,
\begin{equation}
  p_\phi(q_{1:L}\mid c)=\prod_{i=1}^{L}p_\phi\!\left(q_i\mid q_{<i},c\right),
  \label{eq:ar}
\end{equation}
trained with a next-token cross-entropy loss. 
Per-dimension binning yields $L=TD$ tokens, whereas whole-chunk tokenizers such as FAST~\cite{pertsch2025fast} transform the chunk before quantization to shorten $L$. 
At inference the tokens are decoded sequentially and detokenized back to $a$. 
This family reuses the language-model head unchanged, but both the achievable action granularity and the per-chunk decoding cost are dictated by the tokenizer.

\subsubsection{Continuous Generative Modeling}
\label{sec:flow}

Continuous modeling methods generate continuous $a$ and thus avoid quantization. 
Flow matching action modeling~\cite{black2025pi0, intelligence2025pi05,groot2025n1,octo2024octo} models the multimodal action distribution by learning a time-dependent velocity field that transports Gaussian noise to data. 
We adopt the convention in which $t=1$ is noise and $t=0$ is the clean target, so that integrating from $t=1$ to $t=0$ denoises. 
Given a clean chunk $\achunk$ and noise $\boldsymbol{\epsilon}\sim\mathcal{N}(0,I)$, the two are linearly interpolated, with target velocity
\begin{equation}
  \achunk_t = t\,\boldsymbol{\epsilon} + (1-t)\,\achunk,\qquad
  u_t = \boldsymbol{\epsilon} - \achunk.
  \label{eq:interp}
\end{equation}
The head $v_\phi(\achunk_t,t)$ is trained to match $u_t$~\cite{lipman2022flow, liu2023rectified},
\begin{equation}
  \Lflow
  =\mathbb{E}_{t,\boldsymbol{\epsilon}}\big\|v_\phi(\achunk_t,t)-u_t\big\|_2^2 ,
  \label{eq:flow}
\end{equation}
and a chunk is sampled by integrating the field from $t=1$ to $t=0$ in $N$ Euler steps of size $\Delta t=-1/N$,
\begin{equation}
  \achunk_{t+\Delta t}=\achunk_t+\Delta t\,v_\phi(\achunk_t,t).
  \label{eq:ode}
\end{equation}
Since \cref{eq:interp} is linear, predicting the velocity is equivalent to predicting the clean chunk $\hat{\achunk}_0$, through $\hat{\achunk}_0=\achunk_t-t\,v_\phi(\achunk_t,t)$ and $v_t=(\achunk_t-\hat{\achunk}_0)/t$. 
The $x_0$-prediction form, which estimates the clean chunk $\achunk$ at each step, exposes the chunk itself as the object on which a trajectory prior can act.

\begin{figure}
    \centering
    \includegraphics[width=\linewidth]{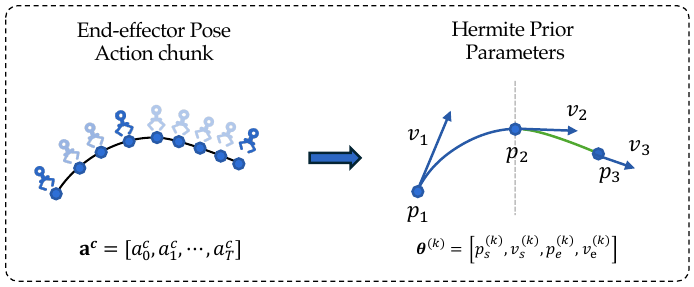}
    \vspace{-3mm}
    \caption{
        Hermite encoding of an action chunk. For each segment $k$, the endpoint positions and velocities $\thetabold^{(k)}=(p_s^{(k)}, v_s^{(k)}, p_e^{(k)}, v_e^{(k)})$ are extracted from the end-effector chunk, yielding a compact trajectory representation.
    }
    \label{fig:hermite_conversion}
\end{figure}

\begin{figure}
    \centering
    \includegraphics[width=\linewidth]{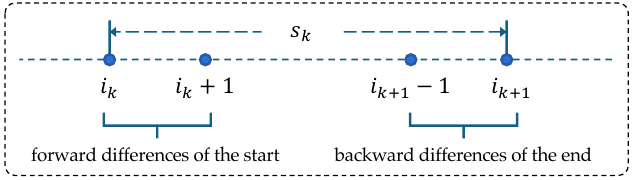}
    \vspace{-3mm}
    \caption{
        One-sided finite-difference estimation of Hermite endpoint velocities. The start and end velocities of segment $k$ are computed from forward and backward differences, respectively, and rescaled by the segment span $s_k$.
    }
    \label{fig:chafen}
\end{figure}
\vspace{1mm}

\subsection{Action Chunks and Hermite Trajectory Operator}
\label{sec:method:operator}

A VLA model emits a fixed-length action chunk
$
    \achunk \in \mathbb{R}^{\horizon \times \actdim},
$
where $\horizon$ denotes the number of future steps and $\actdim$ is the total action dimension. 
To account for distinct physical dynamics, we partition the action space and denote by $\achunk^{c} \in \mathbb{R}^{\horizon \times \actdimc}$ the smooth continuous subspace of the action chunk, formed by selecting the $\actdimc \leq \actdim$ continuously varying channels such as translational and rotational displacement components, while excluding discrete signals such as binary gripper commands.

The central object in our method is a fixed Hermite trajectory operator $\mathcal{D}_{H}(\thetabold) = \herm \thetabold, \label{eq:hermite-operator}$
where $\thetabold \in \mathbb{R}^{4\Ksegs \times \actdimc}$ contains the boundary variables of $\Ksegs$ cubic Hermite segments and $\herm \in \mathbb{R}^{\horizon \times 4\Ksegs}$ is a precomputed basis matrix. 
The output $\mathcal{D}_{H}(\thetabold)$ is a dense trajectory over the horizon in the smooth subspace. 
Thus, $\thetabold$ serves as a structured latent variable for the smooth part of the action chunk, while $\mathcal{D}_{H}$ maps this latent variable back to per-timestep controls, the interface expected by existing VLA decoders.

As shown in~\Cref{fig:hermite_conversion}, an action chunk is first transformed to $K$ segments.
For each segment $k \in \{0,\ldots,\Ksegs-1\}$, the latent variable consists of endpoint positions and velocities,
\begin{equation}
\thetabold^{(k)}=
\left(
p_s^{(k)}, v_s^{(k)}, p_e^{(k)}, v_e^{(k)}
\right)
\in \mathbb{R}^{4\actdimc},
\label{eq:segment-params}
\end{equation}
where subscripts $s$ and $e$ denote the start and end of the segment, and $p$ and $v$ denote position and velocity, respectively. 
With local coordinate $u \in [0,1]$, the segment is decoded as
\begin{equation}
\begin{aligned}
a^{(k)}(u)
=&{}
h_{00}(u)\, p_s^{(k)}
+ h_{10}(u)\, v_s^{(k)} + \\
  &
  h_{01}(u)\, p_e^{(k)}
+ h_{11}(u)\, v_e^{(k)} ,
\end{aligned}
\label{eq:segment-decode}
\end{equation}
where the cubic Hermite basis functions are
\begin{equation}
\begin{aligned}
h_{00}(u) &= 2u^3 - 3u^2 + 1, \ \
h_{10}(u) = u^3 - 2u^2 + u, \\
h_{01}(u) &= -2u^3 + 3u^2, \ \ \ \
h_{11}(u) = u^3 - u^2 .
\end{aligned}
\label{eq:hermite-basis}
\end{equation}
Thus, we obtain the matrix $\herm$ by tabulating these basis functions at the $\horizon$ control timestamps. 
Concretely, let timestamp $i$ fall in segment $k(i)$ with local coordinate $u_i\in[0,1]$. 
Row $i$ of $\herm$ is zero except for the four columns of segment $k(i)$, namely columns $4k(i)$ through $4k(i){+}3$, which hold the basis values:
\begin{equation}
\herm_{i,\,4k(i):4k(i)+3}
=
\big(h_{00}(u_i),\,h_{10}(u_i),\,h_{01}(u_i),\,h_{11}(u_i)\big),
\label{eq:basis-matrix}
\end{equation}
giving $\herm$ a block-banded structure that maps four boundary variables of each segment to the timestamps it spans.
The overall procedure is summarized in \Cref{alg:operator}.

\begin{algorithm}[t]
\caption{Hermite Target Encoding and Trajectory Decoding
(shared by all variants)}
\label{alg:operator}
\begin{algorithmic}[1]
\Require chunk $\mathbf{a}_{\mathrm{gt}}\in\mathbb{R}^{T\times D}$; segment count $K$; precomputed basis $\mathbf{H}\in\mathbb{R}^{T\times 4K}$ (\Cref{eq:basis-matrix})
\Ensure boundary variables $\boldsymbol{\theta}\in\mathbb{R}^{4K\times D_c}$; decoded smooth chunk $\hat{\mathbf{a}}^{c}\in\mathbb{R}^{T\times D_c}$
\State $\mathbf{a}^{c}_{\mathrm{gt}} \gets$ smooth channels of
$\mathbf{a}_{\mathrm{gt}}$
\State $i_k \gets \mathrm{round}\!\big(k\,(T{-}1)/K\big)$ for
$k=0,\dots,K$;\quad $s_k \gets i_{k+1}-i_k$
\For{$k = 0,\dots,K{-}1$}
  \Comment{each segment encoded independently}
  \State $p^{(k)}_s \gets \mathbf{a}^{c}_{\mathrm{gt}}[i_k]$;\quad
         $p^{(k)}_e \gets \mathbf{a}^{c}_{\mathrm{gt}}[i_{k+1}]$
  \State $v^{(k)}_s \gets
  \big(\mathbf{a}^{c}_{\mathrm{gt}}[i_k{+}1]
  -\mathbf{a}^{c}_{\mathrm{gt}}[i_k]\big)\,s_k$
  \Comment{\Cref{eq:fd-velocity}}
  \State $v^{(k)}_e \gets
  \big(\mathbf{a}^{c}_{\mathrm{gt}}[i_{k+1}]
  -\mathbf{a}^{c}_{\mathrm{gt}}[i_{k+1}{-}1]\big)\,s_k$
  \State $\boldsymbol{\theta}^{(k)} \gets
  \big(p^{(k)}_s, v^{(k)}_s, p^{(k)}_e, v^{(k)}_e\big)$
\EndFor
\State \textbf{decode:}\;
$\hat{\mathbf{a}}^{c} \gets \mathcal{D}_H(\boldsymbol{\theta})
= \mathbf{H}\boldsymbol{\theta}$
\Comment{one cached matrix multiply}
\end{algorithmic}
\end{algorithm}


This construction satisfies three core requirements for a VLA-compatible trajectory prior. 
First, it establishes a low-dimensional structural scaffold for the smooth subspace of the action chunk, parameterizing motion with $4\Ksegs\actdimc$ boundary variables rather than $\horizon\actdimc$ unconstrained samples. 
Second, since its parameterization consists directly of endpoint positions and velocities, boundary behavior becomes explicitly controllable at both interior segment knots and inter-chunk replanning seams. 
Third, since the operator $\mathcal{D}_{H}$ is a fixed linear mapping, trajectory-space losses propagate back to $\thetabold$ analytically without specialized optimization procedures.


\vspace{1mm}
\noindent\textbf{Target encoding.}
Given a ground-truth chunk $\achunk_{\mathrm{gt}}$, we first extract its smooth subspace $\achunk_{\mathrm{gt}}^{c}$. 
We then place $\Ksegs{+}1$ knots at the rounded uniform grid $i_k = \operatorname{round}\!\big(k\,(\horizon{-}1)/\Ksegs\big)$, so that adjacent segment spans $s_k = i_{k+1}-i_k$ may differ by one step. 
Each segment is encoded independently into its four boundary variables: the endpoint positions are read directly from the chunk, $p_s^{(k)} = \achunk_{\mathrm{gt}}^{c}[i_k]$ and $p_e^{(k)} = \achunk_{\mathrm{gt}}^{c}[i_{k}+1]$, and the endpoint velocities are the one-sided finite differences at the two ends as shown in~\Cref{fig:chafen}, rescaled by the segment span so that they are expressed in the local coordinate
$u\in[0,1]$,
\begin{equation}
\begin{aligned}
v_s^{(k)} = \big(\achunk_{\mathrm{gt}}^{c}[i_{k}+1] - \achunk_{\mathrm{gt}}^{c}[i_k]\big)\, s_k, \\
v_e^{(k)} = \big(\achunk_{\mathrm{gt}}^{c}[i_{k+1}] - \achunk_{\mathrm{gt}}^{c}[i_{k+1}-1 ]\big)\, s_k .
\label{eq:fd-velocity}
\end{aligned}
\end{equation}
This deterministic encoder is shared across all variants, so differences among variants come from where the Hermite prior enters the VLA decoder rather than from different target construction. 
The encoded targets are $C^0$-exact at interior knots, since the two copies of a shared knot read the same ground-truth frame. 
Tangent ($C^1$) behavior is soft: the two sides of an interior knot carry the backward and forward differences of that frame, whose gap is the local discrete curvature of the demonstration, so the encoding biases the decoder toward tangent-consistent motion without imposing it as a hard constraint.

\begin{figure}
    \centering
    \includegraphics[width=\linewidth]{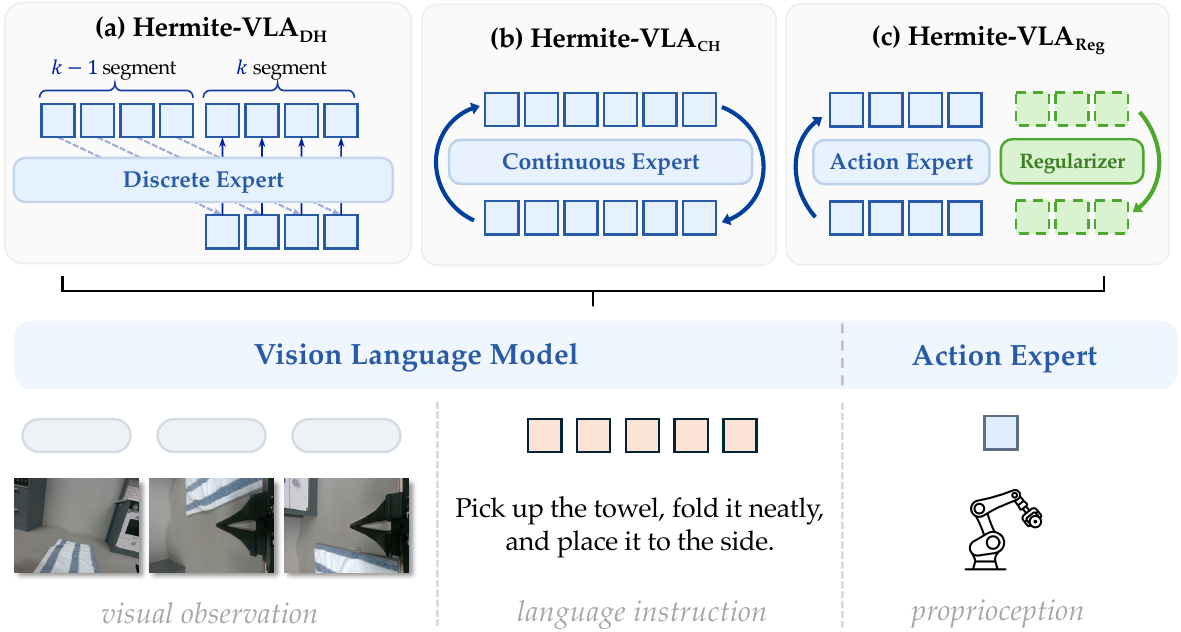}
    \caption{
        Three instantiations of the Hermite trajectory prior.
        (a) Hermite Token (\hdh{}) tokenizes the boundary variables and decodes them segment by segment. 
        (b) Hermite Scaffold (\hch{}) decomposes each clean-action estimate into a Hermite scaffold and per-timestep residuals at every step of the flow matching ODE. 
        (c) Hermite Regularization (\hr{}) attaches a training-only auxiliary Hermite head that shapes the shared features and is discarded at inference.
    }
    \label{fig:injections}
\end{figure}

\subsection{Instantiating the Hermite Trajectory Prior}
\label{sec:method:injections}

According to the action chunk modeling described in~\Cref{sec:ar,sec:flow}, we introduce three variants and aim to explore two nested design questions.
(1)~\textit{Can the same fixed Hermite operator be integrated into both discrete autoregressive and continuous generative action modeling?}
To this end, \hdh{} reformulates discrete action modeling around quantized Hermite boundary variables, whereas \hch{} incorporates a decoded Hermite scaffold into the clean-action prediction of a continuous generative model.
(2)~\textit{Within continuous generative modeling, must the prior explicitly participate in deployed action prediction?}
\hr{} provides an objective-only alternative that retains the original continuous action head and sampling procedure while using Hermite reconstruction as auxiliary supervision.
Thus, \hdh{} and \hch{} demonstrate its applicability across discrete and continuous action modeling, while \hch{} and \hr{} compare an explicit Hermite decomposition of the deployed prediction against an objective-only use of the same operator.
The training and inference procedures are summarized in \Cref{alg:training,alg:inference}, respectively.

\begin{algorithm}[t]
\caption{Training of the Three Hermite-VLA Variants}
\label{alg:training}
\small
\begin{algorithmic}[1]
\Require observation $o$; demonstration chunk
$\mathbf{a}_{\mathrm{gt}}$; basis $\mathbf{H}$; auxiliary weight
$\lambda$
\Statex \textbf{Shared computation}
\State $\mathbf{c} \gets \mathrm{VLM}(o)$
\Comment{backbone evaluated once}
\State $\boldsymbol{\theta}_{\mathrm{gt}} \gets$
\Call{HermiteEncode}{$\mathbf{a}_{\mathrm{gt}}$}
\Comment{Alg.~\ref{alg:operator}}
\State compute $\mathcal{L}$ by the variant's block below;
update parameters by $\nabla\mathcal{L}$
\Statex
\Statex \textbf{\ding{172} Hermite-VLA$_{\mathrm{DH}}$
(segment-level autoregressive)}
\State $\mathbf{q} \gets$ quantize
$\boldsymbol{\theta}_{\mathrm{gt}}$ into $B$ bins \Comment{\Cref{eq:hdh-quant}}
\State $\mathcal{L} \gets \mathcal{L}_{\mathrm{DH}}(\mathbf{q};
\mathbf{c})$ $+$ gripper cross-entropy \Comment{\Cref{eq:hdh-loss}}
\Statex
\Statex \textbf{\ding{173} Hermite-VLA$_{\mathrm{CH}}$
(scaffold $+$ residual)}
\State $t\sim\mathcal{U}(0,1)$,
$\boldsymbol{\epsilon}\sim\mathcal{N}(0,\mathbf{I})$;\quad
$\mathbf{a}_t \gets t\,\boldsymbol{\epsilon}
+ (1{-}t)\,\mathbf{a}_{\mathrm{gt}}$ \Comment{\Cref{eq:interp}}
\State $\mathbf{h} \gets
\mathrm{ActionExpert}(\mathbf{a}_t, t, \mathbf{c})$;\quad
$\boldsymbol{\theta} \gets
\mathrm{MLP}\big(\tfrac{1}{T}\textstyle\sum_{\tau}
\mathbf{h}_{\tau}\big)$;\quad
$\mathbf{r}_{\tau} \gets \mathbf{W}_r \mathbf{h}_{\tau}$
\Comment{\Cref{eq:hch-heads}}
\State $\mathcal{L} \gets
\|\mathbf{H}\boldsymbol{\theta} + \mathbf{r}
- \mathbf{a}^{c}_{\mathrm{gt}}\|_2^2$ $+$ non-smooth-channel loss
\Comment{\Cref{eq:hch-loss}}
\Statex
\Statex \textbf{\ding{174} Hermite-VLA$_{\mathrm{Reg}}$
(auxiliary objective only)}
\State $t$, $\boldsymbol{\epsilon}$, $\mathbf{a}_t$, $\mathbf{h}$
as in lines~6--7
\State $\mathcal{L}_{\mathrm{flow}} \gets
\|v_{\phi}(\mathbf{a}_t,t)
-(\boldsymbol{\epsilon}-\mathbf{a}_{\mathrm{gt}})\|_2^2$;\quad
$\boldsymbol{\theta}_{\mathrm{aux}} \gets
\mathrm{MLP}_{\mathrm{aux}}\big(\tfrac{1}{T}\textstyle\sum_{\tau}
\mathbf{h}_{\tau}\big)$ \Comment{\Cref{eq:flow}; no residual}
\State $\mathcal{L} \gets \mathcal{L}_{\mathrm{flow}}
+ \lambda\,\|\mathbf{H}\boldsymbol{\theta}_{\mathrm{aux}}
- \mathbf{a}^{c}_{\mathrm{gt}}\|_2^2$
\Comment{\Cref{eq:lherm,eq:hr-loss}}
\end{algorithmic}
\end{algorithm}
 
\begin{algorithm}[t]
\caption{Inference of the Three Hermite-VLA Variants}
\label{alg:inference}
\small
\begin{algorithmic}[1]
\Require context $\mathbf{c}=\mathrm{VLM}(o)$; basis $\mathbf{H}$;
denoising steps $N$
\Statex \textbf{\ding{172} Hermite-VLA$_{\mathrm{DH}}$
(segment-level autoregressive)}
\For{$k=0,\dots,K{-}1$}
  \State decode token group $\hat{\mathbf{q}}^{(k)}$ conditioned on
  $\hat{\mathbf{q}}^{(<k)}, \mathbf{c}$
  \Comment{$4D_c$ tokens in parallel}
\EndFor
\State $\hat{\boldsymbol{\theta}} \gets$ dequantize
$\hat{\mathbf{q}}$ \Comment{\Cref{eq:hdh-dequant}}
\State $\hat{\mathbf{a}}^{c} \gets \mathbf{H}\hat{\boldsymbol{\theta}}$;
append gripper-head output
\Statex
\Statex \textbf{\ding{173}/\ding{174} Hermite-VLA$_{\mathrm{CH}}$ and
Hermite-VLA$_{\mathrm{Reg}}$ (flow matching)}
\State $\mathbf{a} \gets \boldsymbol{\epsilon}
\sim\mathcal{N}(0,\mathbf{I})$
\For{$n=1,\dots,N$} \Comment{Euler steps, $\Delta t=-1/N$}
  \State $\hat{\mathbf{a}}_0 \gets$
  \Call{CleanEstimate}{$\mathbf{a}, t_n, \mathbf{c}$}
  \State $\mathbf{a} \gets \mathbf{a}
  + \Delta t\,(\mathbf{a}-\hat{\mathbf{a}}_0)/t_n$
  \Comment{\Cref{eq:hch-ode}}
\EndFor
\Statex \textsc{CleanEstimate}: scaffold-plus-residual
$\mathbf{H}\boldsymbol{\theta}+\mathbf{r}$ concatenated with the
non-smooth channels for \textsc{CH}; the \emph{unchanged} baseline
head for \textsc{Reg} (auxiliary branch discarded)
\end{algorithmic}
\end{algorithm}

\subsubsection{Discrete Modeling with Hermite Tokens}
\hdh{} adapts the Hermite operator to discrete action modeling by treating its boundary variables as action tokens.
An action chunk is first encoded into $\thetabold$ by the shared target encoder. 
Each coordinate $\theta$ is then uniformly quantized into one of $B$ bins over its respective training set range $[\theta_{\min},\theta_{\max}]$: 
\begin{equation} 
q(\theta) = \operatorname{round}\!\left( \frac{ \operatorname{clip}(\theta,\theta_{\min},\theta_{\max}) - \theta_{\min} }{ \theta_{\max}-\theta_{\min} } (B-1) \right). 
\label{eq:hdh-quant} 
\end{equation} 
This yields $4\Ksegs\actdimc$ tokens per chunk, organized into $\Ksegs$ groups of $4\actdimc$ tokens, one group for each Hermite segment. 
We denote the quantized token group of segment $k$ by $\mathbf{q}^{(k)}$. 

Considering the causal relationship between segments, we adopt \emph{segment-level autoregressive} prediction.
Specifically, the vision-language backbone obtains context $\mathbf{c}$ by encoding the observation $o$, and then a lightweight action head predicts the segment groups conditioned on $c$ sequentially.
Within each segment, the $4\actdimc$ coordinate tokens are emitted in parallel: 
\begin{equation} 
\mathcal{L}_{\mathrm{DH}} = -\sum_{k=0}^{\Ksegs-1} \sum_{j=1}^{4\actdimc} \log p_{\phi}\!\left( q_{j}^{(k)} \mid \mathbf{q}^{(<k)},\mathbf{c} \right). 
\label{eq:hdh-loss} 
\end{equation} 
Thus, each segment conditions on all preceding segment groups, while the tokens within the same segment share a common segment context and do not condition on one another. 
The gripper channel is predicted independently at each timestep by a separate output head. 

During training, the logits for all segment groups can be computed in one forward pass under the teacher forcing setting. 
At inference, the segment groups are decoded sequentially. 
After each group is predicted, its token embeddings are supplied as context for the next segment. 
Let $\hat{q}\in\{0, \ldots, B-1\}$ denotes a predicted bin index, and the continuous boundary value $\hat{\theta}$ is then dequantized according to 
\begin{equation} 
\hat{\theta} = \frac{\hat{q}}{B-1} \left( \theta_{\max}-\theta_{\min} \right) + \theta_{\min}, 
\label{eq:hdh-dequant} 
\end{equation} 
and assembled into the predicted boundary variables $\hat{\thetabold}$. 
The smooth action trajectory is reconstructed as 
\begin{equation} 
    \hat{\achunk}^{c} = \mathcal{D}_{H}(\hat{\thetabold}). 
\label{eq:hdh-reconstruct} 
\end{equation} 
The separately predicted gripper channel is combined with the $\hat{\achunk}^{c}$ to form the complete action chunk $\hat{\textbf{a}}$. 

Unlike existing action tokenizers such as FAST~\cite{pertsch2025fast}, \hdh{} assigns each token an explicit trajectory-level meaning: the endpoint position or velocity of a Hermite segment. 
Moreover, its autoregressive loop is confined to the $\Ksegs$ segment groups of a lightweight action head, without inserting the action tokens into the backbone's autoregressive sequence, making the inference more efficient. 
However, discretizing these continuous boundary variables inevitably introduces quantization error, motivating the continuous variants introduced next.

\subsubsection{Continuous Modeling with a Hermite Scaffold}
\label{sec:method:hch}

\hch{} adapts the Hermite operator to continuous generative action modeling under the same $x_0$-prediction formulation as the baseline $\pi_{0.5}$~\cite{intelligence2025pi05}.
At each flow time $t$, the action head decomposes the clean estimate of the smooth action channels into a Hermite scaffold and a per-timestep residual:
\begin{equation}
\hat{\achunk}_{0}^{c}
=
\underbrace{\mathcal{D}_{H}(\thetabold)}_{\text{Hermite scaffold}}
+
\underbrace{\mathbf{r}}_{\text{per-timestep residual}}
=
\herm\thetabold+\mathbf{r}.
\label{eq:hch-output}
\end{equation}

\noindent The residual compensates for high-frequency micro-adjustments and complex local variations that fall outside the expressive range of the low-dimensional Hermite basis.
Formally, let $\mathbf{h}\in\mathbb{R}^{\horizon\times d}$ denote the features produced by the action expert for the noised chunk $\achunk_t$.
The Hermite branch aggregates features over the full horizon, whereas the residual branch produces a separate correction for each timestep:
\begin{equation}
\thetabold
=
\mathrm{MLP}\!\left(
\frac{1}{\horizon}
\sum_{\tau=1}^{\horizon}\mathbf{h}_{\tau}
\right),
\qquad
\mathbf{r}_{\tau}
=
\mathbf{W}_{r}\mathbf{h}_{\tau}.
\label{eq:hch-heads}
\end{equation}
This separates a low-dimensional, chunk-level trajectory component from time-indexed residual corrections.

We adopt a plain $x_0$ reconstruction training objective on the clean action trajectory:
\begin{equation}
\mathcal{L}_{\mathrm{CH}}
=
\mathbb{E}_{t,\boldsymbol{\epsilon}}
\left[
\left\|
\herm\thetabold+\mathbf{r}
-
\achunk_{\mathrm{gt}}^{c}
\right\|_{2}^{2}
\right].
\label{eq:hch-loss}
\end{equation}
Notably, we apply supervision in trajectory space rather than directly to $\thetabold$.
The finite-difference target encoder provides one deterministic estimate of the boundary velocities, whereas trajectory-space supervision requires only that the reconstructed action trajectory match the demonstration.

At inference, the smooth estimate in \Cref{eq:hch-output} is combined with the separately predicted non-smooth channels to form the complete action estimate $\hat{\achunk}_{0}(\achunk_t,t,\mathbf{c})$.
The corresponding velocity $\mathbf{v}_t$ and Euler update are
\begin{equation}
\mathbf{v}_{t}
=
\frac{
\achunk_t-\hat{\achunk}_{0}(\achunk_t,t,\mathbf{c})
}{t},
\qquad
\achunk_{t+\Delta t}
=
\achunk_t+\Delta t\,\mathbf{v}_{t}.
\label{eq:hch-ode}
\end{equation}
Thus, the Hermite prior actively conditions the generative trajectory during every solver step rather than being applied as an external post-processing operation after sampling.

\subsubsection{Continuous Modeling with Hermite Regularization}
\label{sec:method:hr}

While \hch{} explicitly incorporates the Hermite scaffold into every action prediction, a trajectory prior need not necessarily alter the deployed action parameterization.
It serves as a training-time structural objective that encourages the shared action features to capture the global trajectory underlying a dense sequence of per-timestep controls.
Accordingly, \hr{} introduces a low-dimensional Hermite reconstruction task alongside the original action-prediction objective, while leaving the deployed action head and flow-matching sampler unchanged.
Since the auxiliary branch summarizes the full horizon through endpoint positions and velocities, it provides explicit boundary-aware supervision that is absent from the standard dense action prediction objective.

In parallel with the main action head, an auxiliary Hermite branch predicts $\thetabold$ from the shared features.
It adopts the same pooled architecture as the Hermite branch in \Cref{eq:hch-heads}, but omits the per-timestep residual.
The predicted boundary variables are decoded through $\mathcal{D}_{H}$ and supervised against the clean smooth-action trajectory:
\begin{equation}
\Lherm
=
\mathbb{E}_{t,\,\boldsymbol{\epsilon}}
\left[
\left\|
\herm\thetabold
-
\achunk_{\mathrm{gt}}^{c}
\right\|_2^2
\right].
\label{eq:lherm}
\end{equation}
The complete training objective is
\begin{equation}
\mathcal{L}_{\mathrm{Reg}}
=
\Lflow
+
\lambda\Lherm,
\label{eq:hr-loss}
\end{equation}
where $\lambda$ controls the relative weight of the auxiliary Hermite reconstruction.
Since the auxiliary head is used exclusively during training and removed at inference, \hr{} preserves the original action head, sampling procedure, and inference efficiency without incurring any computational overhead.

Although the auxiliary prediction is computed from features conditioned on the noised chunk $\achunk_t$, its target is always the clean trajectory $\achunk_{\mathrm{gt}}^{c}$.
The same trajectory-level target is therefore presented across flow times and noise realizations, encouraging the shared features to retain a consistent description of the underlying motion.
This distinguishes \hr{} from \hch{}: \hch{} uses the Hermite scaffold as an explicit component of the deployed clean-action estimate, whereas \hr{} uses the same operator only to enrich the learning signal.
As a result, \hr{} introduces boundary-aware trajectory supervision without quantization, modifying the runtime output parameterization, or adding inference overhead.

\subsection{Theoretical Analysis of the Hermite Prior}
\label{sec:theory}
 
We introduce three elementary analyses that make precise why the Hermite coordinates are a suitable interface for VLA action learning.
Statements are given for one smoothly varying channel, i.e., $\boldsymbol{\theta}\in\mathbb{R}^{4K}$ and $\mathbf{a}^{c}\in\mathbb{R}^{T}$, and extend to the $D_c$ channels column-wise.
Throughout we assume that every segment contains at least four sampled timestamps (Appendix~A, satisfied at $K{=}2$), which guarantees that $\mathbf H$ has full column rank with smallest singular value $\sigma_{\min}(\mathbf H)>0$.
Each property accounts for one empirical finding of \Cref{sec:experiments}: the supervision-space ablation (\Cref{tab:supervision_ablation}), the bell-shaped dependence on $K$ (\Cref{tab:kscan_ablation}), and the seam reduction (\Cref{tab:realrobot-seam}).
 
\begin{proposition}[Supervision geometry]
\label{prop:metric}
Let $P_{\mathbf{H}} = \mathbf{H}(\mathbf{H}^{\!\top}\mathbf{H})^{-1}\mathbf{H}^{\!\top}$ and, for a target $\mathbf{a}\in\mathbb{R}^{T}$, let $\boldsymbol{\theta}_{\mathrm{LS}}(\mathbf{a}) = (\mathbf{H}^{\!\top}\mathbf{H})^{-1}\mathbf{H}^{\!\top}\mathbf{a}$.
For every $\boldsymbol{\theta}\in\mathbb{R}^{4K}$,
\begin{equation}
\label{eq:metric}
\big\|\mathbf{H}\boldsymbol{\theta}-\mathbf{a}\big\|_2^2
=
\big\|\boldsymbol{\theta}-\boldsymbol{\theta}_{\mathrm{LS}}(\mathbf{a})
\big\|_{\mathbf{H}^{\!\top}\mathbf{H}}^{2}
+
\big\|(\mathbf{I}-P_{\mathbf{H}})\,\mathbf{a}\big\|_2^2 ,
\end{equation}
where $\|\mathbf{x}\|_{M}^{2}=\mathbf{x}^{\!\top}\!M\mathbf{x}$: the trajectory-space losses of \Cref{eq:hch-loss,eq:lherm} supervise the same least-squares target as a $\theta$-space loss, but under the metric $\mathbf{H}^{\!\top}\mathbf{H}$, which weights each boundary coordinate by exactly its influence on the executed trajectory.
\end{proposition}
 
The two metrics exhibit a substantial disparity in scale: the position columns of $\mathbf{H}$ possess elements of order $O(1)$, whereas the velocity columns are strictly bounded by $\max_u\vert{}h_{10}\vert{}=\max_u\vert{}h_{11}\vert{}=4/27$. 
Consequently, an identity-metric loss in $\theta$-space disproportionately penalizes velocity errors relative to their actual impact on the trajectory footprint. 
Reformulating the objective within trajectory space directly mitigates this weighting imbalance.
 
\begin{proposition}[Expressivity and bandwidth of the scaffold]
\label{prop:span}
Suppose the smooth channel samples an underlying trajectory, $\mathbf{a}^{c}_{\mathrm{gt}}[i]=f(i)$ with $f\in C^{2}([0,T{-}1])$, and let
$\boldsymbol{\theta}_{\mathrm{gt}}$ be produced by the target encoder of \Cref{eq:fd-velocity}. 
Then, \emph{(i)} $\mathcal{D}_H(\boldsymbol{\theta}_{\mathrm{gt}})$ is exact at every knot and $\|\mathcal{D}_H(\boldsymbol{\theta}_{\mathrm{gt}})
-\mathbf{a}^{c}_{\mathrm{gt}}\|_{\infty}
\leq \tfrac{3}{8}\, s_{\max}^{2}\,\|f''\|_{\infty}$, where $s_{\max}=\max_k s_k \approx (T{-}1)/K$ and time is measured in control steps;
\emph{(ii)} every trajectory in the span of $\mathcal{D}_H$ has
piecewise-linear acceleration and piecewise-constant jerk taking at
most $K$ distinct values, bounded on segment $k$ by
\begin{equation}
\label{eq:jerk}
\Big|\tfrac{d^{3}a^{(k)}}{dt^{3}}\Big|
\;\leq\;
\frac{12\,\big|p^{(k)}_s-p^{(k)}_e\big|}{s_k^{3}}
+\frac{6\,\big(|\nu^{(k)}_s|+|\nu^{(k)}_e|\big)}{s_k^{2}},
\end{equation}
where $\nu^{(k)}=v^{(k)}/s_k$ are the per-step endpoint velocities.
\end{proposition}
 
The two bounds pull in opposite directions: the approximation error of the target scaffold falls as $K^{-2}$ while the admissible jerk ceiling grows as $K^{3}$, so $K$ trades the low-frequency bias of the prior for fidelity, predicting the interior optimum observed at $K{=}2$. 
Part~\emph{(i)} also shows that the residual $\mathbf{r}$ of \Cref{eq:hch-output} only needs to model a second-order remainder and any content outside $C^{2}$. 
For \textsc{Hermite-VLA}$_{\mathrm{Reg}}$, both parts describe the auxiliary \emph{target}, i.e., the function class toward which the shared features are regularized, rather than the deployed policy.
 
\begin{proposition}[Seam control]
\label{prop:seam}
For any execution window $W\le T$, the incoming side of a handover is exactly the boundary-coordinate pair $\big(\hat p^{(0)}_s,\hat v^{(0)}_s\big)$ of the new chunk, and the outgoing side is a fixed linear functional
$\big(\mathbf r_p^{\top}\hat{\boldsymbol\theta},\, \mathbf r_v^{\top}\hat{\boldsymbol\theta}\big)$ of the previous chunk, given by the rows of $\mathbf H$ and of its derivative counterpart at the last executed timestep. 
When this timestep falls on a segment knot (in particular under full-chunk execution, $W=T$), the outgoing state reduces to boundary coordinates as well. 
Every such error is bounded through $\|\hat{\boldsymbol\theta} -\boldsymbol\theta_{\mathrm{LS}}(\mathbf a)\|_2 \le\|\mathbf H\hat{\boldsymbol\theta}-\mathbf a\|_2 /\sigma_{\min}(\mathbf H)$, up to an absolute factor below $1.03$ in the functional case (Appendix~A). 
Consequently, the seam discontinuity in position and velocity is bounded by the prediction errors of the two adjacent chunks together with the across-seam increment intrinsic to the demonstration, and is therefore directly controlled by the trajectory-space objectives of \Cref{eq:hdh-loss,eq:hch-loss,eq:lherm}.

\end{proposition}

Overall, the distinction from standard flat representations is what the parameters expose: in a flat chunk the seam \emph{position} is directly a coordinate, but the \emph{velocity} of the handover transition is a finite difference bridging consecutive chunks; it is a function of neither individual parameter set and thus goes unmonitored under standard per-chunk behavior cloning objectives. Two caveats qualify the analysis. First, the bound accumulates only per-chunk errors, so the prior narrows seam discontinuities rather than eliminating them. Second, the guarantee concerns the Hermite scaffold, i.e., before residual addition in \hch{}, and enters \hr{} only as an inductive bias imparted through shared features rather than as a hard runtime constraint.
 


\begin{table*}[t]
\centering
    \caption{
        Real-robot tasks and demonstration data. All demonstrations are collected by human teleoperation under end-effector control at $30$\,Hz. ``\#Objects'' denotes the number of manipulated objects per rollout (randomized where a range is given), and ``2-Wrist'' indicates two wrist-mounted cameras.
    }
    \label{tab:realworld-tasks}
    \begin{tabular}{c|l|p{6.cm}|c|c|c|c|c}
    \toprule
    Task & Platform & Instruction & Teleoperation & \#Objects & \#Demos & Duration & Views \\
    \midrule
    1 & \makecell[t]{Franka \\(single-arm)} & Open the pot, put the wooden block into it, and close the pot. & GELLO~\cite{gello2023} & 1--3 & 832 & 6.3\,h & Side/1-Wrist\\
    2 & \makecell[t]{Cybopal \\(single-arm)} & Open the pot, put the wooden block into it, and close the pot. & PICO VR & 1--3 & 1{,}006 & 4.1\,h & Top/1-Wrist \\
    3 & \makecell[t]{Cybopal \\(dual-arm)} & Left arm opens the pot lid, right arm puts the blocks and balls into the pot, and closes the lid. & PICO VR & 1--3 & 805 & 4.6\,h & 2-Wrist/Front\\
    4 & \makecell[t]{ARX \\(dual-arm)} & Pick up the towel, fold it neatly, and place it to the side. & PICO VR & 1 & 723 & 3.2\,h & 2-Wrist/Front \\
    \bottomrule
\end{tabular}
\end{table*}

\section{Experiments}
\label{sec:experiments}

\subsection{Experimental Setup}
\label{sec:experiments:setup}

\noindent\textbf{Simulation Benchmarks.}
We adopt two simulation benchmarks. 
\emph{Standard LIBERO}~\cite{liu2023libero} consists of four task suites, namely Spatial, Object, Goal, and 10, each containing $10$ tasks. 
\emph{LIBERO-plus}~\cite{liberoplus2025} augments these four suites with seven categories of test-time perturbation, namely Camera, Robot, Language, Light, Background, Noise, and Layout, and after filtering and category balancing yields $10{,}030$ perturbed task instances that probe robustness to distribution shift while leaving the nominal goal unchanged.

\vspace{1mm}
\noindent\textbf{Real-world platforms.}
We further evaluate on four real-robot manipulation tasks across three hardware platforms, as shown in \Cref{fig:real-world-platform}, spanning single-arm and dual-arm settings.
\Cref{tab:realworld-tasks} summarizes the tasks, instructions, and demonstration data. Each task is defined as a three-stage sequence (open $\to$ place $\to$ close for Tasks~1--3, pick $\to$ fold $\to$ place for Task~4) so that partial progress can be measured. 
The four tasks are designed to disentangle task difficulty from embodiment: Tasks~1 and~2 share the same instruction on different single-arm platforms, Task~3 moves to a dual-arm configuration of the same platform for coordinated placement, and Task~4 probes deformable-object manipulation on a separate dual-arm platform.

\begin{figure}
    \centering
    \includegraphics[width=\linewidth]{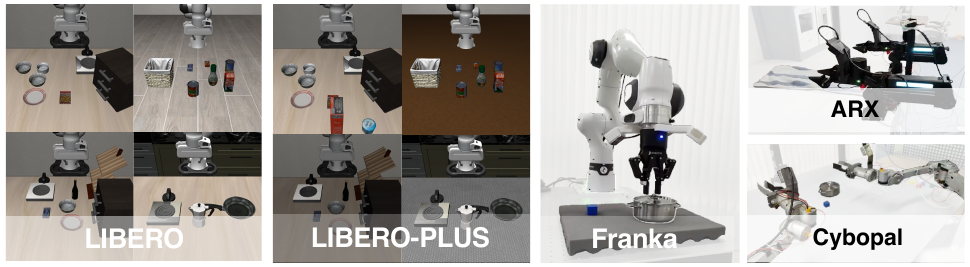}
    \caption{
        Evaluation platforms. Two simulation benchmarks, LIBERO and LIBERO-plus, and three real-robot platforms: a single-arm Franka ($7$\,DoF), a Cybopal arm ($6$\,DoF per arm) used in both single-arm and dual-arm configurations, and a dual-arm ARX platform ($6$\,DoF per arm).
    }
    \label{fig:real-world-platform}
\end{figure}

\vspace{1mm}
\noindent\textbf{Evaluation protocol.}
On standard LIBERO we follow the official protocol and evaluate each task over $50$ rollouts with randomized initial states, giving $500$ episodes per suite, and report the per-suite success rate (SR) and the four-suite average. 
On LIBERO-plus, we report the per-category SR together with the average over all $10{,}030$ instances.
On the real robot, each task is evaluated through $15$ runs with random placement of object positions and random number of objects. 
A rollout is counted as successful only if the final stage is completed, and we additionally report a cumulative stage-wise pass rate, defined as the fraction of the three stages a rollout completes in order. 

\vspace{1mm}
\noindent\textbf{Trajectory-quality metrics.}
While task success rates measure macro-level completion, they fail to capture fine-grained execution dynamics. 
We therefore quantify three trajectory-level metrics, targeting the intra-chunk roughness and inter-chunk discontinuity:

\begin{itemize}
\item \textbf{Jerk} (time-domain smoothness): 
Computed as the third time-derivative of position along trajectories. 
Lower magnitudes signify attenuated mechanical vibrations, reduced tracking errors, and enhanced control stability.
\item \textbf{SPARC} (frequency-domain smoothness)\label{metric:sparc}: 
Evaluated via the Spectral Arc Length of the velocity profile in frequency. Less negative values indicate suppressed high-frequency spectral content, suppressing execution jitter.
\item \textbf{Seam Discontinuity} (inter-chunk boundary consistency): 
Quantified by the step-displacement ratio at chunk handovers relative to nominal execution per step. Lower values denote seamless closed-loop replanning transitions.
\end{itemize}

\vspace{1mm}
\noindent\textbf{Baselines.}
We compare against representative policies from both decoder families and group the results accordingly. 
Among autoregressive token-prediction policies we include OpenVLA~\cite{kim2024openvla}, WorldVLA~\cite{worldvla2025}, and $\pi_0$-FAST~\cite{pertsch2025fast}, whose FAST tokenizer is the direct counterpart of \hdh{} on the same backbone family. 
Among continuous action-generation policies we include Diffusion Policy~\cite{chi2023diffusion}, Octo~\cite{octo2024octo}, DiT Policy~\cite{ditpolicy2024}, OpenVLA-OFT~\cite{openvlaoft2025}, and the flow-matching $\pi_{0.5}$~\cite{intelligence2025pi05}, which is also the backbone shared by all three Hermite-VLA variants. 
For each baseline we report numbers from its original publication where the evaluation protocol matches ours, and otherwise reproduce it according to its original training setting. 

\subsection{Implementation Details}
\label{sec:experiments:impl}
\begin{table*}[t]
\centering
\caption{
    Standard LIBERO success rate (\%) on the four official suites and their average. Params denotes the approximate model size. The top block lists autoregressive token-prediction policies and the bottom block continuous action-generation policies. The best average within each block is in bold.
}
\label{tab:libero_main}
\resizebox{\textwidth}{!}{
\begin{tabular}{l |c |c |c c c c |c}
\toprule
Method 
& Params 
& Training Objective 
& LIBERO-Spatial 
& LIBERO-Object 
& LIBERO-Goal 
& LIBERO-10 
& Avg. \\
\midrule

\multicolumn{8}{l}{\textit{Autoregressive token-prediction policies}} \\
\midrule
OpenVLA~\cite{kim2024openvla} 
& $\sim$7B 
& AR CE 
& 84.7 & 88.4 & 79.2 & 53.7 & 76.5 \\

WorldVLA~\cite{worldvla2025} 
& $\sim$7B 
& AR CE
& 87.6 & 96.2 & 83.4 & 60.0 & 81.8 \\

$\pi_0$-FAST~\cite{pertsch2025fast} 
& $\sim$3.3B 
& AR CE 
& 96.4 & 96.8 & 88.6 & 60.2 & 85.5 \\

Hermite-VLA$_{\mathrm{DH}}$ (Ours) 
& $\sim$3.3B 
& AR CE 
& 96.0 & 98.6 & 96.4 & 90.4 & 95.4 \\

\midrule
\multicolumn{8}{l}{\textit{Continuous action-generation policies}} \\
\midrule
Diffusion Policy~\cite{chi2023diffusion} 
& 157M 
& Diffusion 
& 78.3 & 92.5 & 68.3 & 50.5 & 72.4 \\

Octo~\cite{octo2024octo} 
& 93M 
& Diffusion 
& 78.9 & 85.7 & 84.6 & 51.1 & 75.1 \\

DiT Policy~\cite{ditpolicy2024} 
& 334M 
& Diffusion 
& 84.2 & 96.3 & 85.4 & 63.8 & 82.4 \\

OpenVLA-OFT~\cite{openvlaoft2025} 
& $\sim$7B 
& L1 regression 
& 97.6 & 98.4 & 97.9 & 94.5 & 97.1 \\

$\pi_{0.5}$~\cite{intelligence2025pi05} 
& $\sim$3.3B 
& Flow matching 
& 98.4 & 96.0 & 97.8 & 91.4 & 95.9 \\

Hermite-VLA$_{\mathrm{CH}}$ (Ours) 
& $\sim$3.3B 
& Flow matching 
& 99.2 & \textbf{99.2} & 96.4 & 96.0 & 97.7 \\

Hermite-VLA$_{\mathrm{Reg}}$ (Ours) 
& $\sim$3.3B 
& Flow matching + Hermite reg. 
& \textbf{99.6} & \textbf{99.2} & \textbf{99.0} & \textbf{96.8} & \textbf{98.7} \\

\bottomrule
\end{tabular}
}
\end{table*}

\noindent\textbf{Backbones.}
For our models, all three variants are instantiated on $\pi$-series~\cite{intelligence2025pi05} backbone, a PaliGemma vision-language tower~\cite{beyer2024paligemma} with an action expert.
\hdh{} replaces the $\pi_0$-FAST tokenizer with the hermite tokenizer; 
\hch{} restructures the clean-action prediction of the flow-matching head into the scaffold-plus-residual form;
\hr{} attaches the auxiliary head and leaves the deployed head unchanged. 
All runs (the baselines $\pi_0$-FAST, $\pi_{0.5}$, \hdh{}, \hch{}, and \hr{}) share one training framework and the same training setting, while differ only in the head and the loss, so their comparison is controlled.

\vspace{1mm}
\noindent\textbf{Trajectory operator.}
All three variants use $\Ksegs=2$ segments on the rounded uniform grid. 
On LIBERO and LIBERO-plus, $\horizon=10, W=5$, $\actdim=7$, and $\actdimc=6$, excluding the binary gripper channel. 
On the real-robot platforms, $\horizon=50,W=20$ at a $30$\,Hz control rate, with end-effector control throughout: $\actdim=7$, $\actdimc=6$ for the single-arm Franka and Cybopal setups, and $\actdim=14$, $\actdimc=12$ for the dual-arm Cybopal and ARX setups, excluding one binary gripper channel per arm. 
The basis matrix $\herm$ is precomputed once and cached on the accelerator.
The non-smooth channels are predicted by the separate tail and concatenated with the decoded smooth trajectory. 
The auxiliary and replacement Hermite heads use a two-layer MLP of hidden width $512$,  and \hdh{} quantizes each Hermite coordinate into $256$ bins.

\vspace{1mm}
\noindent\textbf{Training parameters.}
All variants are initialized from the open-source $\pi$-series (the discrete variant from $\pi_0$-FAST and continuous variants from $\pi_{0.5}$) base checkpoint, with the newly introduced Hermite or auxiliary heads randomly initialized, and fine-tuned with AdamW under a cosine learning-rate schedule warming up over $1$k steps to a peak of $5{\times}10^{-5}$ and decaying to $3{\times}10^{-5}$, batch size $256$, and quantile action normalization. 
We adopt the $\lambda=10$ for \hr{} training and all variants train for $30$k steps with $8{\times}$ H100 GPUs.
At inference, the continuous variants sample with $N{=}10$ denoising steps.

\begin{table*}[t]
\centering
\caption{
    LIBERO-plus success rate (\%) under the seven test-time perturbation categories, together with the instance-weighted average (Avg.) over all $10{,}030$ instances. Params denotes the approximate model size. The top block lists autoregressive token-prediction policies and the bottom block continuous action-generation policies. The best result in each column within each block is in bold.
}
\label{tab:liberoplus_main}
\setlength{\tabcolsep}{7.3pt}
\begin{tabular}{l|c|ccccccc|c}
\toprule
Method & Params & Camera & Robot & Language & Light & Background & Noise & Layout & Instance Avg.\\
\midrule
\multicolumn{10}{l}{\textit{Autoregressive token-prediction policies}} \\
\midrule
OpenVLA~\cite{kim2024openvla} & $\sim$7B & 0.8 & 3.5 & 23.0 & 8.1 & 34.8 & 15.2 & 28.5 & 15.6 \\
WorldVLA~\cite{worldvla2025} & $\sim$7B & 0.1 & 27.9 & 41.6 & 43.7 & 19.8 & 10.9 & 38.0 & 25.3 \\
$\pi_{0}$-FAST~\cite{pertsch2025fast} & $\sim$3.3B & 65.1 & 21.6 & 61.0 & 73.2 & 73.2 & 74.4 & 68.8 & 61.6 \\
Hermite-VLA$_\mathrm{DH}$ (Ours) & $\sim$3.3B & 54.3 & 36.9 & 70.4 & 92.8 & 86.4 & 78.8 & 77.6 & 69.4 \\
\midrule
\multicolumn{10}{l}{\textit{Continuous action-generation policies}} \\
\midrule
OpenVLA-OFT~\cite{openvlaoft2025} & $\sim$7B & 56.4 & 31.9 & 79.5 & 88.7 & 93.3 & 75.8 & 74.2 & 69.6 \\
$\pi_{0.5}$~\cite{intelligence2025pi05} & $\sim$3.3B & 75.8 & 79.4 & 83.3 & 95.5 & 95.0 & 89.6 & 87.0 & 85.7 \\
Hermite-VLA$_\mathrm{CH}$ (Ours) & $\sim$3.3B & 78.7 & 83.3 & \textbf{97.7} & 87.0 & 85.9 & 73.9 & 86.8 & 85.0 \\
Hermite-VLA$_\mathrm{Reg}$ (Ours) & $\sim$3.3B & \textbf{89.2} & \textbf{85.4} & 87.3 & \textbf{97.7} & \textbf{96.7} & \textbf{93.8} & \textbf{89.9} & \textbf{90.9} \\
\bottomrule
\end{tabular}
\end{table*}

\subsection{Main Results}
\label{sec:experiments:main}

\subsubsection{Simulation Results}
\label{sec:experiments:main:libero}
\noindent\textbf{Standard LIBERO.}
We summarize the quantitative performance across all four LIBERO suites in \Cref{tab:libero_main}.
For autoregressive policies, \hdh{} achieves a $95.4\%$ average success rate, substantially surpassing $\pi_0$-FAST ($85.5\%$). 
For continuous policies, \hr{} and \hch{} attain $98.7\%$ and $97.7\%$ average success rates respectively, both outperforming the $\pi_{0.5}$ baseline ($95.9\%$). The gains are most substantial on the challenging LIBERO-10 suite, where \hr{} elevates the baseline from $91.4\%$ to $96.8\%$.
\hdh{} follows the same trend, gaining $30.2\%$ ($60.2\%\!\rightarrow\!90.4\%$) over $\pi_0$-FAST on this suite.
These results strongly align with our analysis in \Cref{sec:intro}. 
Long-horizon tasks in LIBERO-10 involve frequent replanning seams over fixed execution windows. 
While flat representations provide no trajectory-level mechanisms to regularize motion across these chunk boundaries, Hermite variants expose boundary states via structured tokens, scaffold decomposition, or auxiliary regularization.
This substantial boost confirms that explicit boundary parameterization effectively prevents error accumulation during extended closed-loop execution.

\vspace{1mm}
\noindent\textbf{LIBERO-plus.}
\label{sec:experiments:main:libero_plus}
\Cref{tab:liberoplus_main} reports per-category and average success rates under the seven environmental perturbations. 
Among continuous policies, \hr{} achieves the best overall average of $90.9\%$ and ranks first in six of the seven categories, improving over the $\pi_{0.5}$ baseline ($85.7\%$) in every category.
Its largest gains occur under Camera and Robot perturbations, by $13.4\%$ and $6.0\%$, respectively, whereas the improvements are smaller for appearance-only shifts such as Light ($+2.2\%$) and Background ($+1.7\%$).
This profile is consistent with an action-side prior that primarily structures how motion is generated rather than how visual appearance is represented.
\hch{} reaches $85.0\%$ average SR, leading on geometric perturbations but 
\begin{table}[t]
\centering
\caption{
    Real-robot success rate (\%) on four manipulation tasks, $15$ rollouts each with randomized object placement. The cumulative stage-wise pass rate indicates \emph{where} rollouts fail: S1 and S2 denote completion of the first and second stages, and SR denotes overall success (completion of the final stage). The best SR per task is in bold.
}
\label{tab:realrobot}
\small
\setlength{\tabcolsep}{8pt}
\begin{tabular}{l|l|cc|c}
\toprule
Task & Method & S1 & S2 & SR \\
\midrule
\multirow{5}{*}{Task~1}
 & $\pi_0$-FAST & 93.3 & 60.0 & 60.0 \\
 & $\pi_{0.5}$ & 100.0 & 93.3 & 86.7 \\
 & Hermite-VLA$_\mathrm{DH}$ & 100.0 & 80.0 & 80.0 \\
 & Hermite-VLA$_\mathrm{CH}$ & 100.0 & 86.7 & 86.7 \\
 & Hermite-VLA$_\mathrm{Reg}$ & 100.0 & 100.0 & \textbf{100.0} \\
\midrule
\multirow{5}{*}{Task~2}
 & $\pi_0$-FAST & 53.3 & 20.0 & 20.0 \\
 & $\pi_{0.5}$ & 86.7 & 26.7 & 26.7 \\
 & Hermite-VLA$_\mathrm{DH}$ & 66.7 & 33.3 & 33.3 \\
 & Hermite-VLA$_\mathrm{CH}$ & 93.3 & 60.0 & 60.0 \\
 & Hermite-VLA$_\mathrm{Reg}$ & 80.0 & 66.7 & \textbf{66.7} \\
\midrule
\multirow{5}{*}{Task~3}
 & $\pi_0$-FAST & 86.7 & 46.7 & 40.0 \\
 & $\pi_{0.5}$ & 93.3 & 53.3 & 46.7 \\
 & Hermite-VLA$_\mathrm{DH}$ & 93.3 & 46.7 & 46.7 \\
 & Hermite-VLA$_\mathrm{CH}$ & 100.0 & 80.0 & 80.0 \\
 & Hermite-VLA$_\mathrm{Reg}$ & 100.0 & 93.3 & \textbf{93.3} \\
\midrule
\multirow{5}{*}{Task~4}
 & $\pi_0$-FAST & 86.7 & 86.7 & 80.0 \\
 & $\pi_{0.5}$ & 93.3 & 93.3 & 93.3 \\
 & Hermite-VLA$_\mathrm{DH}$ & 93.3 & 93.3 & 86.7 \\
 & Hermite-VLA$_\mathrm{CH}$ & 100.0 & 100.0 & \textbf{100.0} \\
 & Hermite-VLA$_\mathrm{Reg}$ & 100.0 & 100.0 & \textbf{100.0} \\
\midrule
\multirow{5}{*}{Avg}
 & $\pi_0$-FAST & 80.0 & 53.4 & 50.0 \\
 & $\pi_{0.5}$ & 93.3 & 66.7 & 63.4 \\
 & Hermite-VLA$_\mathrm{DH}$ & 88.3 & 60.8 & 61.7 \\
 & Hermite-VLA$_\mathrm{CH}$ & 98.3 & 81.7 & 81.7 \\
 & Hermite-VLA$_\mathrm{Reg}$ & 95.0 & 90.0 & \textbf{90.0}\\
\bottomrule
\end{tabular}
\vspace{-3mm}
\end{table}

falling below the baseline on Light, Background, and Noise.
The comparison suggests that, when the pretrained action predictor already models dense actions effectively, Hermite structure is more beneficial as complementary supervision.
\hr{} preserves the original action head while adding an auxiliary trajectory objective, yielding more consistent gains across perturbations.
Among autoregressive policies, \hdh{} reaches $69.4\%$ average SR, compared with $61.6\%$ for $\pi_0$-FAST, showing that the prior also benefits discrete action modeling.

\subsubsection{Real-world Results}
\label{sec:experiments:main:realrobot}
\Cref{tab:realrobot} reports per-task success on the four real-robot tasks. 
Among continuous policies, \hr{} reaches the best average SR at $90.0\%$, against $63.4\%$ for the $\pi_{0.5}$ baseline and $81.7\%$ for \hch{}. 
The gain is largest on two tasks, Task~2 ($26.7\% \to 66.7\%$) and Task~3 ($46.7\% \to 93.3\%$); Task~1 rises to $100\%$ and Task~4, already near ceiling, stays there. 
The autoregressive pair preserves its simulation ordering on hardware, \hdh{} clearly ahead of $\pi_0$-FAST, though both trail the continuous methods. 
The stage-wise breakdown locates where the difference arises: the first stage (opening the lid or picking up the object) is solved at or near $100\%$ by all methods (except $\pi_0$-FAST), whereas the final SR is governed almost entirely by the middle stage, the precise placement or fold, whose pass rate on Task~2 and Task~3 rises from $26.7\%$ and $53.3\%$ under the baseline to $66.7\%$ and $93.3\%$ under \hr{}. 
On Task~2 this improvement holds even though the first-stage rate is lower than the baseline, indicating that the gain comes from the execution of the manipulation itself rather than from reaching the manipulation more reliably.
We relate this to the executed motion in \Cref{sec:experiments:smoothness}.

\subsection{Ablation Studies}
\label{sec:experiments:ablations}

\textbf{Injection strength ($\lambda$).}
\begin{table}[t]
\centering
\caption{
    Injection-strength sensitivity for \hr{} on standard LIBERO across auxiliary-loss weights $\lambda$ (at $K=2$). The best average success rate and the lowest median jerk are in bold.
}
\label{tab:lambda_scan}
\setlength{\tabcolsep}{6.5pt}
\begin{tabular}{l| cccc|c|c} 
\toprule
$\lambda$ & Spatial & Object & Goal & 10 & Avg & Median Jerk $\downarrow$ \\
\midrule
$0$   & 98.4 & 96.0 & 97.8 & 91.4 & 95.9 & 0.000460 \\
$1$   & 98.6 & 98.8 & 97.4 & 94.8 & 97.4 & 0.000443 \\
$5$   & \textbf{99.8} & \textbf{99.8} & 98.4 & 94.6 & 98.2 & 0.000422 \\
$10$  & 99.6 & 99.2 & \textbf{99.0} & \textbf{96.8} & \textbf{98.7} & \textbf{0.000410} \\
$20$  & 99.0 & 99.2 & 97.2 & 94.6 & 97.5 & 0.000418 \\
\bottomrule
\end{tabular}
\end{table}
To evaluate the sensitivity of \hr{} to the auxiliary regularization weight $\lambda$, we conduct an ablation scan on standard LIBERO as shown in \Cref{tab:lambda_scan}.
At $\lambda{=}0$, the objective reverts to the unregularized baseline. 
As $\lambda$ increases, the average success rate exhibits a characteristic bell-shaped profile, peaking at $98.7\%$ at $\lambda{=}10$ before slightly declining to $97.5\%$ at $\lambda{=}20$.
This slight drop indicates that an overly aggressive auxiliary loss competes with the primary flow-matching objective.
Crucially, as $\lambda$ scales from $0$ to $10$, the median execution jerk drops monotonically from $0.00046$ to $0.00041$, confirming that trajectory-level regularization directly enhances physical motion smoothness. 
While performance on shorter-horizon suites saturates early above $97\%$, the long-horizon LIBERO-10 suite proves most sensitive to regularization, surging from $91.4\%$ to $96.8\%$ ($+5.4\%$ absolute gain). 
This improvement underscores the vital role of boundary-aware trajectory priors in extended closed-loop manipulation tasks.

\vspace{1mm}
\noindent\textbf{Segment count ($\Ksegs$).}
\begin{table}[t]
\centering
\caption{
    Segment-count ablation for \hr{} on standard LIBERO as the number of Hermite segments $K$ is varied (at $\lambda=10$). The best average success rate and the lowest median jerk are in bold.
}
\label{tab:kscan_ablation}
\setlength{\tabcolsep}{7pt}
\begin{tabular}{l| cccc|c | c}
\toprule
$K$ & Spatial & Object & Goal & 10 & Avg. & Median Jerk $\downarrow$ \\
\midrule
$1$ & 96.2 & 94.0 & 96.4 & 88.0 & 93.7 & 0.000451 \\
$2$ & \textbf{99.6} & \textbf{99.2} & \textbf{99.0} & \textbf{96.8} & \textbf{98.7} & \textbf{0.000410} \\
$3$ & 97.0 & 95.6 & 98.2 & 95.4 & 96.6 & 0.000426 \\
$4$ & 96.0 & 98.2 & 97.8 & 91.0 & 95.8 & 0.000439 \\
\bottomrule
\end{tabular}
\end{table}
We evaluate the sensitivity of \hr{} to the number of Hermite segments $\Ksegs$ at a fixed weight of $\lambda{=}10$.
As shown in \Cref{tab:kscan_ablation}, the four-suite average success rate exhibits a clear bell-shaped trend that peaks at $\Ksegs{=}2$ ($98.7\%$).
Under-parameterization ($\Ksegs{=}1$) restricts expressiveness, making it insufficient to capture complex multi-stage motion chunks ($93.7\%$).
Conversely, over-parameterization ($\Ksegs{=}3, 4$) increases representation capacity but degrades success rates to $96.6\%$ and $95.8\%$, respectively.
This trend is further corroborated by physical execution metrics, as $\Ksegs{=}2$ achieves the lowest median jerk, whereas higher segment counts increase trajectory jitter. 
This performance degradation is most pronounced on the long-horizon LIBERO-10 benchmark, where $\Ksegs{=}2$ outperforms $\Ksegs{=}1$ by a substantial $+8.8\%$ margin.
Given a fixed horizon of $\horizon{=}10$, allocating excessive knot points reduces each segment span to only a few discrete timesteps, tightly coupling boundary velocities to adjacent frames and diluting the low-frequency smoothness prior.

\begin{table}[t]
\caption{
    Supervision-design ablation for \hch{} on standard LIBERO. Rows~1--2 apply supervision directly in $\thetabold$-space under least-squares and finite-difference targets, respectively; Rows~3--4 supervise the reconstructed trajectory, where Row~4 additionally incorporates per-timestep residuals (the full \hch{}). The best average is in bold.
}
\label{tab:supervision_ablation}
\centering
\small
\setlength{\tabcolsep}{3.5pt}
\resizebox{\linewidth}{!}{
\begin{tabular}{l |l |c| c c c c |c}
\toprule
Supervision & Target & Res. & Spatial & Object & Goal & 10 & Avg. \\
\midrule
$\theta$-space & least-squares          & \ding{55} & 94.6 & 92.6 & 94.4 & 85.0 & 91.5 \\
$\theta$-space & finite-diff            & \ding{55} & 98.2 & 97.2 & 96.2 & 91.6 & 96.4 \\
Trajectory     & (implicit LS)          & \ding{55} & 98.0 & 98.2 & 95.4 & 93.2 & 96.2    \\
Trajectory     & --                     & \ding{51} & \textbf{99.2} & \textbf{99.2} & \textbf{96.4} & \textbf{96.0} & \textbf{97.7} \\
\bottomrule
\end{tabular}}
\end{table}
\vspace{1mm}
\noindent\textbf{Continuous-head design choices.}
To ablate the individual effects of Hermite scaffold and residual components, we systematically evaluate supervision representations and target configurations for \hch{} in \Cref{tab:supervision_ablation}.
Within $\theta$-space, supervising boundary variables against global least-squares targets (Row 1) underperforms the local finite-difference encoder (Row 2) by $4.9\%$ ($91.5\%$ vs. $96.4\%$). 
This drop indicates that global least-squares optimization diffuses fitting errors across the entire horizon, weakening the local physical grounding of boundary parameters. 
Shifting supervision directly to trajectory space without residuals (Row 3) implicitly recovers the least-squares objective while effectively re-weighting the optimization metric through the normal matrix $\herm^{\top}\herm$. 
The resulting $4.7\%$ performance surge ($91.5\%$ to $96.2\%$) demonstrates that the failure in Row 1 stems from unweighted coordinate representation in $\theta$-space rather than the target formulation itself. 
Notably, trajectory-space supervision matches the performance of the best engineered $\theta$-space target while eliminating manual target encoding, and even outperforms it on the long-horizon LIBERO-10 suite ($93.2\%$ vs. $91.6\%$). 
Finally, incorporating per-timestep residuals (Row 4) provides a clean single-factor boost to $97.7\%$, confirming that while the Hermite scaffold imposes an essential structural prior, per-timestep residuals are necessary to capture high-frequency details outside the cubic spline span.

\vspace{1mm}
\noindent\textbf{Boundary basis.}
\begin{table}[t]
\centering
\caption{
    Boundary-basis ablation on standard LIBERO at $K=2$. The best average success rate and the lowest median jerk are in bold.
}
\label{tab:boundary_basis}
\setlength{\tabcolsep}{4.4pt}

\begin{tabular}{l|cccc|c|c}
\toprule
Basis & Spatial & Object & Goal & 10 & Avg. & Median Jerk $\downarrow$ \\
\midrule
Polynomial & 98.6 & 97.8 & 97.8 & 92.4 & 96.7 & 0.000447 \\
Bernstein  & 98.0 & 97.0 & 98.8 & 93.8 & 96.9 & 0.000439 \\
B-spline   & 98.2 & 98.8 & 98.4 & 95.0 & 97.6 & 0.000422 \\
\textbf{Ours} & \textbf{99.6} & \textbf{99.2} & \textbf{99.0} & \textbf{96.8} & \textbf{98.7} & \textbf{0.000410} \\
\bottomrule
\end{tabular}
\vspace{-1mm}
\end{table}
To verify whether the performance gain stems from the specific coordinate parameterization rather than general polynomial trajectory smoothing, we compare the cubic Hermite operator against three standard alternatives under identical evaluation settings~\cite{deboor1978practical, farin2002curves}: a global cubic polynomial, a per-segment Bernstein basis, and a uniform cubic B-spline.
While all candidates define fixed linear mappings into low-order polynomial trajectory spaces, the key distinction lies in the coordinate representation driven by the action head. 
Hermite parameters represent boundary positions and velocities directly, yielding physically meaningful quantities that naturally align with action normalization statistics. 
In contrast, Bernstein control points relate to boundary kinematics only indirectly through linear basis transformations, whereas global polynomial and B-spline parameterizations impose structural constraints that restrict expressiveness over the trajectory family. 
As reported in \Cref{tab:boundary_basis}, our Hermite representation achieves the highest average success rate ($98.7\%$) and the lowest execution jerk ($0.000410$).
This advantage highlights an optimization-level alignment achieved by explicitly mapping latent coordinates to seam-governing physical states, an effect most pronounced on the long-horizon LIBERO-10 benchmark where boundary continuity dictates success.

\begin{table}[t]
\centering
\caption{
    Inference cost on a single NVIDIA H100 $80$\,GB GPU, averaged over $1{,}000$ end-to-end action-chunk generations from a fixed LIBERO observation (two camera views).
}
\label{tab:inference_cost}
\setlength{\tabcolsep}{6pt}
\begin{tabular}{l| c c c}
\toprule
Variant & Latency (ms) $\downarrow$ & FPS $\uparrow$ & Peak Mem. (GB) $\downarrow$ \\
\midrule
$\pi_0$-FAST                   & 238.2 & 4.2  & 7.57 \\
Hermite-VLA$_{\mathrm{DH}}$    & 28.9  & 34.6 & 6.66 \\
$\pi_{0.5}$                    & 48.4  & 20.7 & 6.63 \\
Hermite-VLA$_{\mathrm{CH}}$    & 49.4  & 20.2 & 6.64 \\
Hermite-VLA$_{\mathrm{Reg}}$   & 48.6  & 20.6 & 6.85 \\
\bottomrule
\end{tabular}
\end{table}
 
\vspace{1mm}
\noindent\textbf{Inference cost.}
To verify that introducing structured trajectory priors does not compromise real-time control efficiency, we evaluate the mean per-chunk generation latency and peak GPU memory on a single NVIDIA H100 GPU in \Cref{tab:inference_cost}. 
Across the flow-matching variants, execution overhead remains virtually indistinguishable: baseline $\pi_{0.5}$ ($48.4$ ms), \hr{} ($48.6$ ms), and \hch{} ($49.4$ ms) all operate within a narrow latency window while maintaining a tight peak memory footprint between $6.63$ and $6.85$ GB. 
This empirical result validates the zero-cost deployment property established in \Cref{sec:method:hr}, confirming that stripping the auxiliary Hermite head at inference introduces no operational penalty.
For autoregressive generation, \hdh{} achieves the lowest latency ($28.9$ ms/$34.6$ FPS) by executing the transformer backbone only once and subsequently decoding compact trajectory tokens via an internal $\Ksegs{=}2$ loop without ODE integration. 
In contrast, the token-by-token $\pi_0$-FAST baseline suffers substantial latency ($238.2$ ms) due to repeated backbone forward passes for each individual action token. 
Overall, these benchmarks confirm that the proposed Hermite prior imposes no computational or memory overhead at deployment across both continuous and discrete VLA architectures.

\subsection{Smoothness and Trajectory-Quality Analysis}
\label{sec:experiments:smoothness}
Two policies may attain comparable task success rates while exhibiting qualitatively distinct execution kinematics. 
Therefore, we investigate two fundamental complementary questions: does the Hermite prior systematically refine the geometric \emph{profile} of executed motion, and does this refinement yield tangible deployment advantages? 
To address this, we evaluate end-effector trajectory traces across both simulated LIBERO rollouts and real-world physical executions. 
We explicitly restrict this comparative evaluation to the continuous policy family ($\pi_{0.5}$, \hch{}, \hr{}).
By evaluating policies that operate within the same continuous flow-matching sampling framework, this controlled setup isolates motion smoothness as a function of the Hermite injection scheme, while avoiding confounding noise introduced by discrete tokenization and autoregressive decoding dynamics.

\vspace{1mm}
\noindent\textbf{Global smoothness.}
\begin{table}[t]
\centering
\caption{
    Spectral Arc Length (SPARC) of end-effector translational speed on LIBERO, averaged over $500$ rollouts per cell. Less negative is smoother.
}
\label{tab:sparc}
\resizebox{\linewidth}{!}{
\begin{tabular}{l c c c c c}
\toprule
Variant & Spatial & Object & Goal & 10 & Avg. \\
\midrule
$\pi_{0.5}$ & $-2.293$ & $-2.307$ & $-2.396$ & $-2.733$ & $-2.432$ \\
\hch{}  & $-2.261$ & $-2.257$ & $-2.336$ & $-2.717$ & $-2.403$ \\
\hr{}   & $\mathbf{-2.235}$ & $\mathbf{-2.220}$ & $\mathbf{-2.313}$ & $\mathbf{-2.658}$ & $\mathbf{-2.356}$ \\
\bottomrule
\end{tabular}}
\end{table}
To quantitatively assess holistic trajectory quality across long execution horizons, we evaluate end-effector translational speed trajectories using the SPARC metric across $2{,}000$ LIBERO rollouts (\Cref{tab:sparc}).
Less negative SPARC values indicate higher spectral smoothness.
Across all benchmark suites, \hr{} consistently achieves the highest overall smoothness, attaining a four suite average of $-2.356$ compared to $-2.403$ for \hch{} and $-2.432$ for $\pi_{0.5}$.
The strict hierarchy (\hr{} $>$ \hch{} $>$ $\pi_{0.5}$) holds universally across all four benchmark suites, confirming that Hermite regularization provides consistent spectral smoothness gains from short horizon manipulation to extended task execution.

\begin{figure}[t]
\centering
\includegraphics[width=0.95\columnwidth]{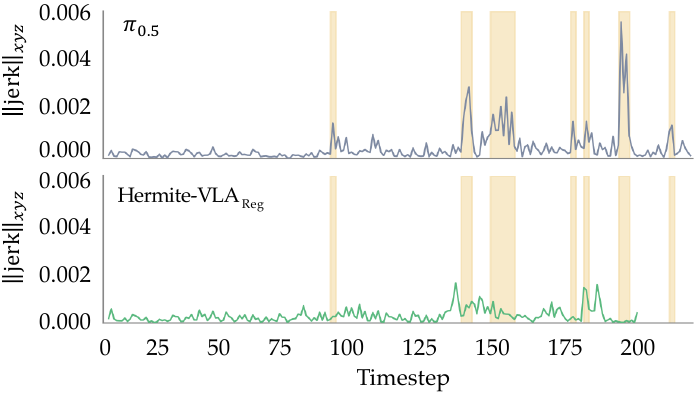}
\caption{
    Per-step jerk magnitude $\|\mathrm{jerk}_{xyz}(t)\|$ along one LIBERO-10 episode that $\pi_{0.5}$ (top) and \hr{} (bottom) both complete from the same initial state. Shaded bands mark the jerk spikes of $\pi_{0.5}$, reproduced at the same positions in both panels. 
}
\label{fig:seam_jerk}
\end{figure}

\vspace{1mm}
\noindent\textbf{Temporal localization of motion roughness.} 
Aggregate metrics can obscure \emph{when} motion instability occurs during execution. 
To resolve this temporal dimension, \Cref{fig:seam_jerk} traces the per-step end-effector jerk magnitude across a representative LIBERO-10 episode, evaluated on both $\pi_{0.5}$ and \hr{} under identical initial conditions. 
The jerk profile of $\pi_{0.5}$ reveals sustained control instability during task execution, marked by an elevated noise floor and sharp transient spikes that reach nearly $0.006$.
As highlighted by the shaded intervals, these regions mark the periods where $\pi_{0.5}$ experiences severe motion bursts.
Across these exact same shaded regions, \hr{} suppresses both the baseline signal noise and the transient discontinuities, maintaining peak jerk spikes below $0.002$ while reaching task completion in fewer execution steps.
This temporal trace explains the global gains in \Cref{tab:sparc}: Hermite prior reduces overall jerk by eliminating sharp, localized motion spikes, rather than simply dampening an already smooth trajectory across the board.

\vspace{0.1mm}
\noindent\textbf{Smoothness tracks task outcome.}
\begin{table}[t]
\centering
\caption{
    Trajectory smoothness versus task outcome on LIBERO. For each variant, the mean jerk magnitude $\|\mathrm{jerk}_{xyz}\|$ is reported over successful versus failed rollouts, together with the number of failed episodes. Failed rollouts are markedly less smooth than successful ones across all variants; \hr{} is the smoothest on both subsets and fails least.
}
\label{tab:smooth_success}
\small
\begin{tabular}{l |c c c}
\toprule
Variant & jerk (success) & jerk (failure) & \# failed \\
\midrule
$\pi_{0.5}$ & 0.00046 & 0.00070 & 82 \\
Hermite-VLA$_{\mathrm{CH}}$     & 0.00047 & 0.00052 & 46 \\
Hermite-VLA$_{\mathrm{Reg}}$    & \textbf{0.00040} & \textbf{0.00048} & \textbf{26} \\
\bottomrule
\end{tabular}
\end{table}

To verify that the observed smoothness gains are not merely an artifact of higher success rates, such as survivor bias, we disentangle trajectory quality from task outcome.
As shown in \Cref{tab:smooth_success}, we divide all rollouts into successful and failed execution subsets and find that failed episodes exhibit consistently higher jerk across all policy variants.
Crucially, \hr{} achieves the lowest jerk within both the successful ($0.00040$) and failed ($0.00048$) subsets, confirming that its kinematic smooth execution holds independently of task outcome.
While this tight correlation does not fully resolve causal directionality~\cite{ratliff2009chomp, kalakrishnan2011stomp, schulman2014motion}, namely whether physical motion roughness causes manipulation failures or out of distribution states trigger erratic control, both perspectives converge on a common insight.
Specifically, \hr{} achieves the highest smoothness across both subsets while incurring the lowest failure count ($26$ failed episodes, compared to $82$ for $\pi_{0.5}$ and $46$ for \hch{}).
This indicates that Hermite regularization simultaneously eliminates physical control instability and prevents policy drift away from the training distribution.
Consequently, improved task success and enhanced trajectory smoothness are not independent findings, but joint manifestations of a fundamentally more robust control policy.

\begin{figure*}[t]
\centering
\includegraphics[width=\linewidth]{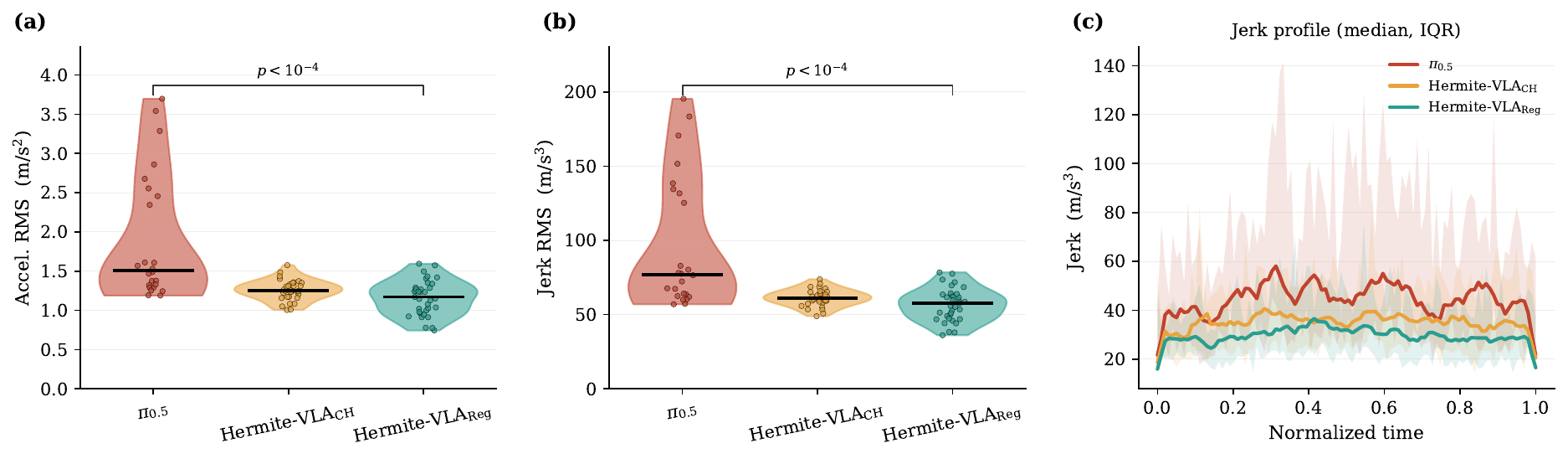}
\vspace{-4mm}
\caption{
    Episode-level smoothness on the real robot. (a) Acceleration RMS and (b) jerk RMS per episode. Both Hermite variants are significantly smoother than $\pi_{0.5}$ ($p<10^{-4}$). (c) Median jerk over normalized episode time with inter-quartile band: $\pi_{0.5}$ is the highest with the widest spread, \hr{} the lowest and most consistent.
}
\label{fig:smoothness_violin}
\end{figure*}

\begin{figure*}[t]
\centering
\includegraphics[width=0.72\linewidth]{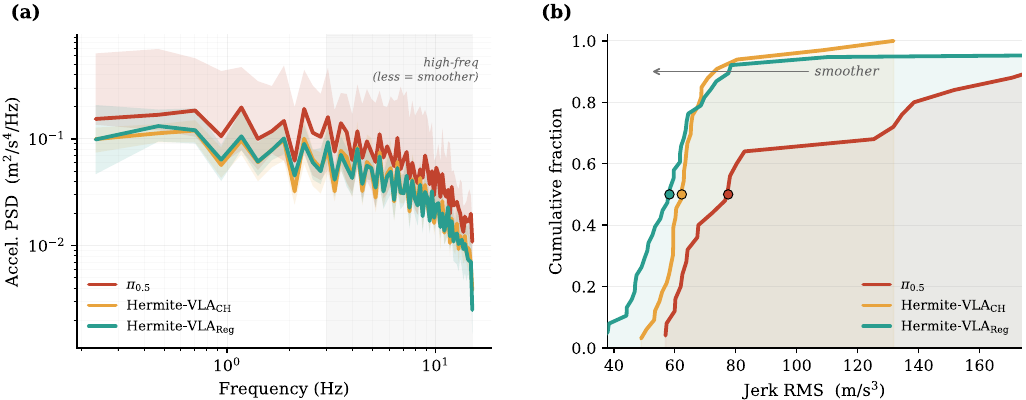}
\caption{Real-robot frequency-domain smoothness. (a) Welch acceleration power spectral density; high-frequency content manifests as jitter on hardware. (b) Cumulative distribution of per-episode jerk RMS.}
\label{fig:spectrum_cdf}
\end{figure*}

\begin{figure*}[t]
\centering
\includegraphics[width=0.9\linewidth]{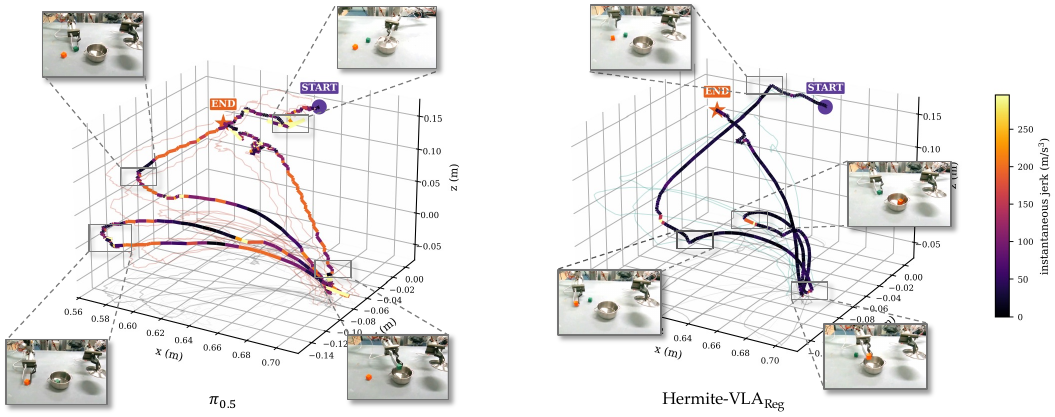}
\caption{
    Executed end-effector paths on Task~3 for $\pi_{0.5}$ (left) and \hr{} (right) in the same scene. Thin lines show all rollouts. The bold line is one representative rollout, colored by instantaneous jerk ($\mathrm{m/s^3}$).
} 
\label{fig:eef_paths}
\end{figure*}

\begin{table}[t]
\centering
\caption{
    Real-robot seam discontinuity ($\Delta a_{\mathrm{seam}} = \|a_t - a_{t-1}\|$) on commanded end-effector translation at replanning handovers ($W{=}20$, $T{=}50$), reported as median (IQR). ``rel.'' is the seam median relative to the $\pi_{0.5}$ baseline, and $\rho = \Delta a_{\mathrm{seam}} / \Delta a_{\mathrm{interior}}$ measures the discontinuity ratio of handover steps to interior trajectory steps.
}
\label{tab:realrobot-seam}
\resizebox{\linewidth}{!}{
\begin{tabular}{l|l|ccc}
\toprule
Task & Variant & $\Delta a_{\text{seam}} (\Delta a_\mathrm{interior})$ & rel. & $\rho$ \\
\midrule
\multirow{3}{*}{Task 2}
& $\pi_{0.5}$  & $0.0056\;(0.0036)$ & $1.00\times$ & $7.6$ \\
& \hch{}       & $0.0041\;(0.0035)$ & $0.73\times$ & $6.3$ \\
& \hr{} & $\mathbf{0.0040\;(0.0033)}$ & $\mathbf{0.72\times}$ & $7.0$ \\
\midrule
\multirow{3}{*}{Task 3}
& $\pi_{0.5}$  & $0.0059\;(0.0076)$ & $1.00\times$ & $8.6$ \\
& \hch{}       & $0.0034\;(0.0040)$ & $0.58\times$ & $5.7$ \\
&\hr{} & $\mathbf{0.0028\;(0.0027)}$ & $\mathbf{0.48\times}$ & $\mathbf{5.5}$ \\
\bottomrule
\end{tabular}}
\end{table}

\vspace{1mm}
\noindent\textbf{Seam discontinuity on real-world platforms.}
To directly quantify physical control discontinuities at replanning boundaries and complement our holistic hardware kinematic analysis, we evaluate acceleration transients at replanning handovers across precision placement tasks (\Cref{tab:realrobot-seam}).
Since velocity discontinuities at handovers manifest as sudden acceleration spikes, we measure the seam magnitude $\|a_t - a_{t-1}\|$ at every handover step and compare it against interior trajectory steps.
We summarize three main insights from this boundary evaluation.
First, every policy variant exhibits significantly higher control discontinuity at replanning handovers than during interior execution ($\rho$ ranging from $5.5$ to $8.6$), which is consistent with the observation in simulation analysis.
Second, Hermite regularization substantially mitigates this handover jump, with \hr{} reducing the relative seam median to $0.72\times$ on Task $2$ and $0.48\times$ on Task $3$ compared to $\pi_{0.5}$ baseline.
Since Hermite regularization simultaneously smooths interior motion, the ratio $\rho$ conservatively understates these absolute boundary gains, paving the way for the holistic kinematic smoothness achieved across the entire rollout.
Third, while Hermite regularization narrows the seam spike, it does not completely eliminate it, which aligns with its formulation as a soft prior whose decoder is trained against boundary quantities governing handover transitions.

\vspace{1mm}
\noindent\textbf{Kinematic gains transfer to hardware.}
Beyond overall task success, we specifically analyze the kinematic smoothness of real robot rollouts to evaluate whether Hermite regularization reduces physical joint wear and mechanical vibration on hardware as follows.

(1) To quantify overall motion quality and its temporal persistence during execution, we evaluate the time domain acceleration and jerk metrics across all rollouts.
As shown in \Cref{fig:smoothness_violin}(a) and \Cref{fig:smoothness_violin}(b), both Hermite variants achieve significantly lower acceleration RMS and jerk RMS than $\pi_{0.5}$ ($p < 10^{-4}$, two sided Mann Whitney U test), with \hr{} recording the lowest median jerk RMS.
This advantage holds throughout execution: the time normalized jerk profiles in \Cref{fig:smoothness_violin}(c) show that baseline $\pi_{0.5}$ maintains an elevated median jerk baseline, whereas Hermite variants suppress jerk magnitudes across the entire rollout.

(2) To uncover how control smoothness reflects spectral energy distribution, we examine the frequency characteristics of the executed motions.
As depicted in \Cref{fig:spectrum_cdf}(a), frequency analysis confirms that Hermite regularization attenuates high frequency acceleration power spectral density (PSD), significantly eliminating rapid oscillations that induce physical jitter.
Consequently, as shown in \Cref{fig:spectrum_cdf}(b), the cumulative distribution of jerk RMS shifts substantially toward lower values across rollouts.

(3) To intuitively understand how these kinematic gains translate to spatial execution quality, we map Cartesian end-effector pose trajectories across all successful rollouts with instantaneous jerk values.
While $\pi_{0.5}$ exhibits severe jerk spikes exceeding $200$\,m/s$^3$ along representative rollouts and substantial path dispersion across all trials (\Cref{fig:eef_paths}), \hr{} consistently traces clean, low jerk 3D trajectories toward identical targets.

\begin{figure}
    \centering
    \includegraphics[width=0.95\linewidth]{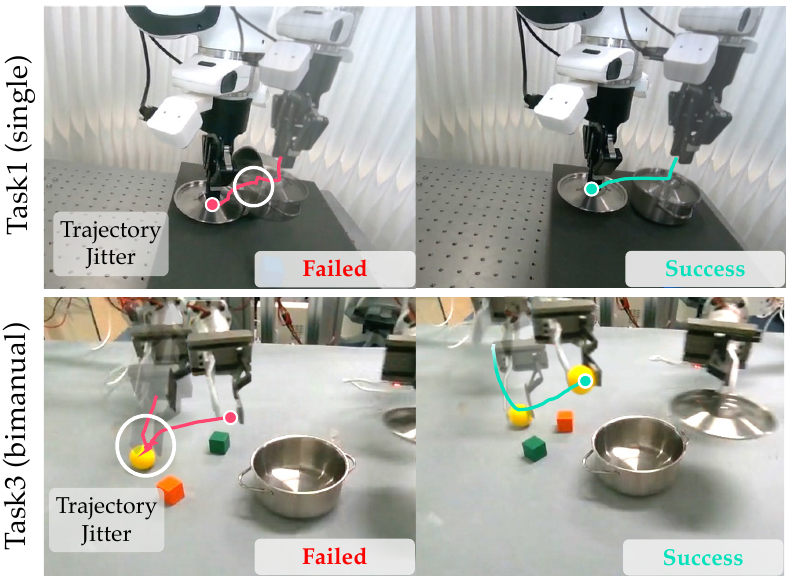}
    \caption{
        Effect of trajectory regularity on real-robot tasks.
        $\pi_{0.5}$ (left) exhibits end-effector jitter near the object, perturbing the gripper--object alignment and resulting in failed grasps in both single-arm and bimanual tasks.
        \hr{} (right) produces a smoother approach trajectory, maintains stable alignment before contact and move successfully.
    }
    \label{fig:real-trajectory}
\end{figure}

\vspace{1mm}
\noindent\textbf{Qualitative Trajectory Analysis.}
To intuitively analyze how trajectory smoothness affects physical execution and task completion efficiency, we visualize representative end-effector trajectories on real-world platforms (\Cref{fig:real-trajectory}) and in simulation (\Cref{fig:sim-trajectory}).

On real-world platforms, these kinematic defects directly determine execution success or failure.
As shown in \Cref{fig:real-trajectory}, $\pi_{0.5}$ baseline suffers from severe trajectory jitter immediately before object engagement in both single arm and bimanual precision tasks.
This sudden jitter disrupts gripper alignment, resulting in missed grasps and physical task failure.
By enforcing Hermite regularization, the policy maintains a smooth, continuous approach that preserves precise alignment through contact, directly converting trajectory regularity into reliable physical success.

For the simulation, detailed temporal traces reveal how $\pi_{0.5}$ suffers from pervasive local oscillations, detours, and execution delays.
As illustrated in \Cref{fig:sim-trajectory}, during initial object approach (\textcircled{1}), $\pi_{0.5}$ generates jagged zigzag waypoints, whereas \hr{} follows a direct, continuous path.
This discrepancy compounds over extended execution (\textcircled{2}): by $t=102$, \hr{} has already transported the object directly over the container, whereas $\pi_{0.5}$ baseline lags substantially behind due to cumulative trajectory detours.
Furthermore, during final object release (\textcircled{3}, \textcircled{4}), $\pi_{0.5}$ executes redundant looping re corrections before completion.
In contrast, \hr{} eliminates path redundancy, completing the lower task by $t=88$, which is twice as fast as $\pi_{0.5}$ ($t=176$).

\begin{figure*}
    \centering
    \includegraphics[width=0.92\linewidth]{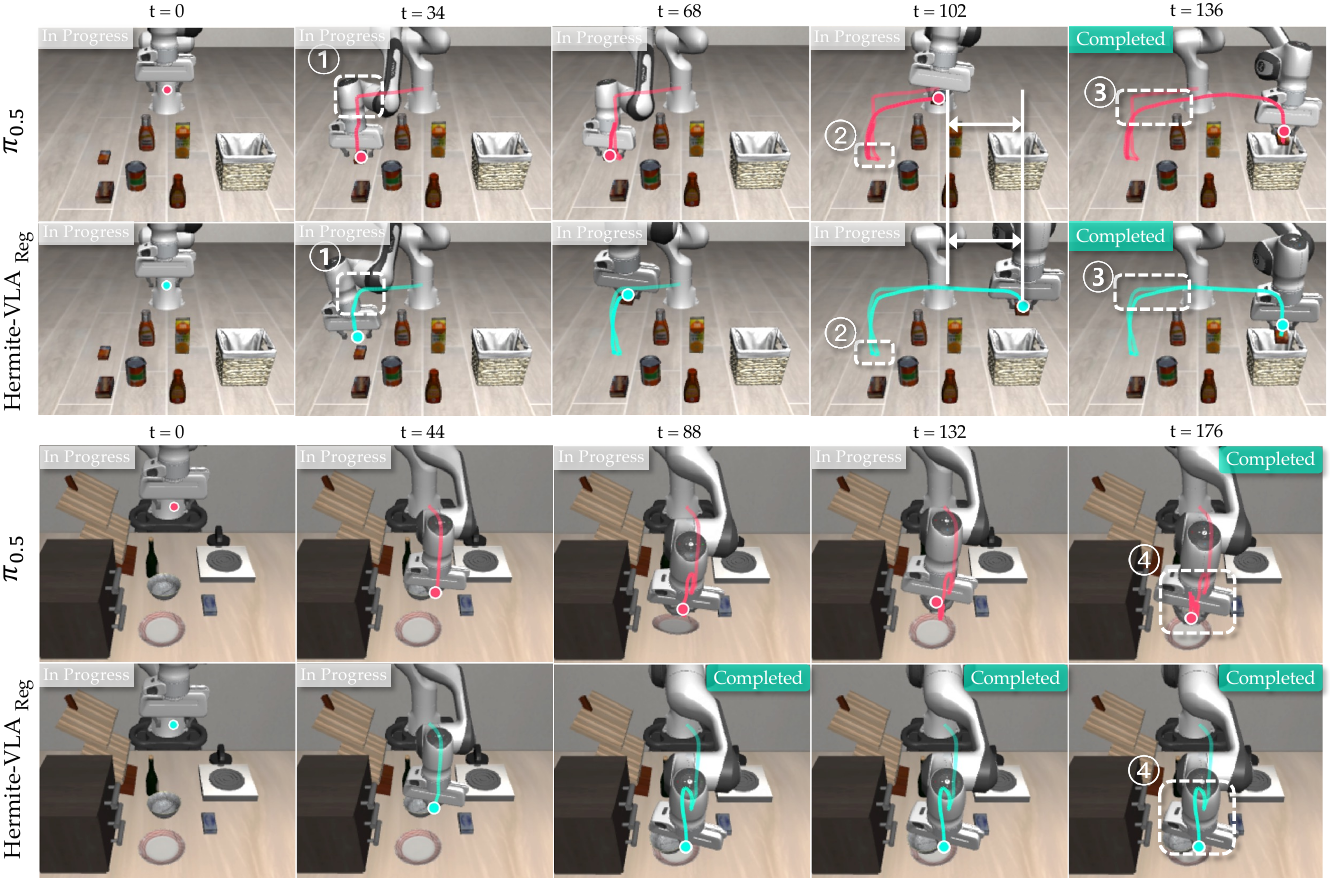}
    \caption{
        Qualitative comparison of end-effector trajectories in simulation.
        The $\pi_{0.5}$ baseline exhibits local oscillations, abrupt corrections, and longer paths during object approach and placement, as highlighted by the numbered regions.
        Hermite Regularization produces smoother and more direct trajectories and reaches successful task completion earlier.
    }
    \label{fig:sim-trajectory}
\end{figure*}

\section{Limitation and Future Work}
Our study highlights two key methodological trade-offs.
First, Hermite regularization models trajectory smoothness as a soft inductive bias rather than strict constraints, preserving policy expressivity at the expense of absolute trajectory continuity.
Second, the proposed prior focuses on action-sequence structure, operating independently of vision-side representations and of non-smooth contact events.
Future research will explore tighter couplings between trajectory priors and visual perception, task adaptive chunk segmentation, and extensions toward highly dynamic, contact rich manipulation tasks.

\section{Conclusion}
\label{sec:conclusion}

In this paper, we comprehensively revisit the representation of action chunks, demonstrating that they are far more effectively modeled as continuous trajectories rather than flat sequences of discrete controls.
To this end, we introduce Hermite trajectory priors, utilizing piecewise cubic Hermite segments to explicitly capture the underlying endpoint positions and velocities that govern intra chunk smoothness and inter chunk continuity.
Built upon a single differentiable operator, this prior is instantiated at three distinct integration levels: \hdh{} tokenizes boundary variables for segment level autoregressive decoding, \hch{} decomposes clean action estimates into a Hermite scaffold with per timestep residuals, and \hr{} applies the operator purely as an auxiliary training objective.
Across LIBERO, LIBERO plus, and four real world manipulation tasks, incorporating Hermite priors consistently boosts execution success across both discrete and continuous VLA architectures.
Notably, the training only \hr{} variant delivers the largest empirical gains at zero inference cost while producing measurably smoother motion and more reliable physical rollouts.
These findings demonstrate that trajectory structure serves as a powerful, lightweight inductive bias without disrupting pretrained representations, offering a promising paradigm for future VLA design.

\bibliographystyle{IEEEtran}
\bibliography{refs}

\newpage
\setcounter{page}{1}
\appendices

\section{Proofs of Propositions 3.1--3.3}
\label{app:proofs}
 
This appendix proves the three propositions of Sec.~3.4 and reports the
spectrum of the Hermite operator underlying them
(Table~\ref{tab:spectrum}). Proposition~3.1 identifies the metric under
which a trajectory-space loss supervises the boundary coordinates;
Proposition~3.2 bounds the approximation error and the bandwidth of the
scaffold; Proposition~3.3 bounds the seam discontinuity, with the
position and velocity components treated separately, because only the
latter requires the operator to be invertible on its range.
 
One convention holds throughout. The action space is delta end-effector
control on every platform, so $\ba^{c}$ is a sequence of per-step
displacements and the boundary coordinates $(p,v)$ of a Hermite segment
correspond physically to endpoint velocities and accelerations of the
executed motion. We retain the position/velocity notation for
consistency with the spline literature; Remark~\ref{rem:phys} states
the translation to physical derivatives explicitly.
 
\subsection{Notation and standing assumptions}
 
All statements are given for one smoothly varying channel, i.e.,
$\bth\in\mathbb{R}^{4K}$ and $\ba^{c}\in\mathbb{R}^{T}$, and extend to
the $D_c$ channels column-wise. Throughout, $f$ denotes the underlying
displacement profile of the demonstration as a function of step index,
so that $\|f''\|_\infty$ bounds the jerk of the executed motion in
control-step units.
 
Recall from Eq.~(9) that timestamp $i\in\{0,\dots,T-1\}$ is assigned to
a unique segment $k(i)$ with local coordinate
$u_i=(i-i_{k(i)})/s_{k(i)}\in[0,1]$: timestamps $i\in[i_k,i_{k+1})$
belong to segment $k$, and the final timestamp $T{-}1$ belongs to
segment $K{-}1$ with $u=1$. Row $i$ of $\bH$ is supported on the four
columns of segment $k(i)$ only. Hence the rows and the columns of $\bH$
are both partitioned by segment: after a permutation of rows,
\begin{equation*}
\bH \;=\; \operatorname{diag}\!\big(\bH_0,\dots,\bH_{K-1}\big),
\qquad \bH_k\in\mathbb{R}^{m_k\times 4},
\end{equation*}
where $m_k$ is the number of timestamps assigned to segment $k$, with
$\textstyle\sum_k m_k = T$ ($m_k=s_k$ for $k<K{-}1$ and
$m_{K-1}=s_{K-1}{+}1$ under the assignment above).
 
\begin{assumption}[identifiability]
\label{asm:ident}
Every segment contains at least four sampled timestamps, $m_k\ge 4$ for
all $k$.
\end{assumption}
 
Assumption~\ref{asm:ident} necessarily requires $4K\le T$, since
$\sum_k m_k=T$. It is satisfied by the operative configuration $K=2$ of
the main paper for both $T=10$ (segments of four and six timestamps)
and $T=50$ (at least $24$). It is used only by Propositions~3.1
and~3.3, through Lemma~\ref{lem:rank}; Proposition~3.2 requires no rank
assumption. Note that $m_k\ge 4$ is also \emph{necessary} for $\bH_k$
to have full column rank, since $\operatorname{rank}(\bH_k)\le m_k$.
 
\begin{lemma}[elementary properties of the Hermite basis]
\label{lem:basis}
For the basis functions of Eq.~(8) and $u\in[0,1]$:
\begin{enumerate}
\item[(a)] \emph{(Endpoint cardinality.)}
\begin{equation*}
\begin{array}{c|cccc}
 & h_{00} & h_{10} & h_{01} & h_{11}\\ \hline
\text{value at } u{=}0 & 1 & 0 & 0 & 0\\
\text{derivative at } u{=}0 & 0 & 1 & 0 & 0\\
\text{value at } u{=}1 & 0 & 0 & 1 & 0\\
\text{derivative at } u{=}1 & 0 & 0 & 0 & 1
\end{array}
\end{equation*}
\item[(b)] \emph{(Affine reproduction.)} $h_{00}(u)+h_{01}(u)\equiv 1$, and
\begin{align*}
h_{00}(u)-(1-u) &= -\big(h_{01}(u)-u\big)\\
&= h_{10}(u)+h_{11}(u) = u(2u-1)(u-1).
\end{align*}
\item[(c)] \emph{(Extrema.)}
$0\le h_{10}(u)=u(1-u)^2\le \tfrac{4}{27}$ (maximum at $u=\tfrac13$),
$-\tfrac{4}{27}\le h_{11}(u)=-u^2(1-u)\le 0$ (minimum at $u=\tfrac23$),
and
\begin{equation*}
h_{10}(u)+|h_{11}(u)| = u(1-u)^2+u^2(1-u) = u(1-u)\le\tfrac14 .
\end{equation*}
Moreover $\max_u|h_{00}|=\max_u|h_{01}|=1$.
\item[(d)] \emph{(Second derivatives.)}
$h_{00}''=12u-6$, $h_{10}''=6u-4$, $h_{01}''=-12u+6$, $h_{11}''=6u-2$.
\item[(e)] \emph{(Third derivatives.)}
$h_{00}'''\equiv 12$, $h_{10}'''\equiv 6$, $h_{01}'''\equiv -12$,
$h_{11}'''\equiv 6$.
\item[(f)] \emph{(First-derivative bounds.)}
$|h_{00}'(u)|=|h_{01}'(u)|=6u(1-u)\le\tfrac32$ and
$|h_{10}'(u)|,\,|h_{11}'(u)|\le 1$ for all $u\in[0,1]$.
\end{enumerate}
\end{lemma}
 
\begin{proof}
Direct computation from Eq.~(8).
\end{proof}
 
\begin{lemma}[rank and spectrum of $\bH$]
\label{lem:rank}
Under Assumption~\ref{asm:ident}, every block $\bH_k$ has full column
rank; consequently $\bH$ has full column rank and
$\sigma_{\min}(\bH)=\min_{k}\sigma_{\min}(\bH_k)>0$.
\end{lemma}
 
\begin{proof}
By Lemma~\ref{lem:basis}(a) the $4\times 4$ matrix obtained by applying
the four endpoint functionals (value/derivative at $u=0,1$) to
$\{h_{00},h_{10},h_{01},h_{11}\}$ is the identity, so the four basis
functions are linearly independent and form a basis of the space
$\mathcal{P}_3$ of cubic polynomials. Write
$[\,h_j(u_i)\,]_{i\le m_k,\,j\le 4}=\mathbf{V}\mathbf{B}$, where
$\mathbf{V}\in\mathbb{R}^{m_k\times 4}$ is the monomial matrix
$\mathbf{V}_{i\ell}=u_i^{\ell}$, $\ell=0,\dots,3$, and
$\mathbf{B}\in\mathbb{R}^{4\times 4}$ is the (invertible) change of
basis from monomials to the Hermite basis. The nodes $\{u_i\}$ within
one segment are pairwise distinct and $m_k\ge 4$; hence
$\mathbf{V}\mathbf{c}=0$ would force a nonzero cubic to vanish at
$m_k\ge 4$ distinct points, so $\mathbf{V}$ has full column rank and so
does $\bH_k$. The block-diagonal structure implies that the singular
values of $\bH$ are the union of those of the blocks.
\end{proof}
 
\begin{table}[t]
\centering
\caption{Spectrum of the Hermite operator $\bH$ under the rounded
uniform grid of Eq.~(9). Here $m_k$ is the number of sampled timestamps
per segment, and the column headed A.\,1 marks whether
Assumption~\ref{asm:ident} holds. Rows with $\sigma_{\min}=0$ are
column-rank deficient, so the boundary coordinates are not identifiable
from the decoded trajectory. The operative configuration $K=2$ is in
bold.}
\label{tab:spectrum}
\small
\setlength{\tabcolsep}{3.5pt}
\begin{tabular}{r r l r c r r r}
\toprule
$K$ & $4K$ & $m_k$ & $\min_k m_k$ & A.\,1
& $\sigma_{\min}$ & $\sigma_{\max}$ & $\mathrm{cond}$\\
\midrule
\multicolumn{8}{l}{\emph{$T=10$ (LIBERO, LIBERO-plus)}}\\
1 & 4 & $(10)$ & 10 & \checkmark & 0.0937 & 2.2607 & 24.1\\
\textbf{2} & \textbf{8} & $\mathbf{(4,6)}$ & \textbf{4} & \checkmark
  & \textbf{0.0189} & \textbf{1.7475} & \textbf{92.7}\\
3 & 12 & $(3,3,4)$ & 3 & --- & 0 & 1.4230 & $\infty$\\
4 & 16 & $(2,2,3,3)$ & 2 & --- & 0 & 1.3303 & $\infty$\\
\midrule
\multicolumn{8}{l}{\emph{$T=50$ (real robot)}}\\
1 & 4 & $(50)$ & 50 & \checkmark & 0.1669 & 5.0666 & 30.4\\
\textbf{2} & \textbf{8} & $\mathbf{(24,26)}$ & \textbf{24} & \checkmark
  & \textbf{0.1056} & \textbf{3.6517} & \textbf{34.6}\\
3 & 12 & $(16,17,17)$ & 16 & \checkmark & 0.0838 & 2.9645 & 35.4\\
4 & 16 & $(12,12,13,13)$ & 12 & \checkmark & 0.0698 & 2.5976 & 37.2\\
5 & 20 & $(10,10,9,10,11)$ & 9 & \checkmark & 0.0564 & 2.3716 & 42.1\\
6 & 24 & $(8,8,8,9,8,9)$ & 8 & \checkmark & 0.0510 & 2.1723 & 42.6\\
\bottomrule
\end{tabular}
\end{table}
 
Table~\ref{tab:spectrum} reports the spectrum of $\bH$ for both
horizons used in the main paper. The knot grid uses round-half-to-even,
the default of the Python/NumPy/JAX \texttt{round} primitive, so that
$i_1=4$ at $T=10$, $K=2$; the script that produces the table is
released with the code.
 
\subsection{Proof of Proposition 3.1 (supervision geometry)}
 
Let $\mathbf{P}_{\bH}=\bH(\bH^\top\bH)^{-1}\bH^\top$ denote the
orthogonal projector onto $\operatorname{col}(\bH)$ and
$\bth_{\mathrm{LS}}(\ba)=(\bH^\top\bH)^{-1}\bH^\top\ba$ (well defined
by Lemma~\ref{lem:rank}). Decompose
$\ba=\mathbf{P}_{\bH}\ba+(\mathbf I-\mathbf{P}_{\bH})\ba$ and note
$\mathbf{P}_{\bH}\ba=\bH\bth_{\mathrm{LS}}(\ba)$. Then
\begin{equation*}
\bH\bth-\ba
=\underbrace{\bH\big(\bth-\bth_{\mathrm{LS}}\big)}_{\in\,\operatorname{col}(\bH)}
\;-\;\underbrace{(\mathbf I-\mathbf{P}_{\bH})\,\ba}_{\perp\,\operatorname{col}(\bH)} .
\end{equation*}
The two terms are orthogonal, so by the Pythagorean identity
\begin{align*}
\big\|\bH\bth-\ba\big\|_2^2
&=\big\|\bH(\bth-\bth_{\mathrm{LS}})\big\|_2^2
+\big\|(\mathbf I-\mathbf{P}_{\bH})\ba\big\|_2^2\\
&=\big\|\bth-\bth_{\mathrm{LS}}\big\|_{\bH^\top\bH}^2
+\big\|(\mathbf I-\mathbf{P}_{\bH})\ba\big\|_2^2 ,
\end{align*}
which is Eq.~(21). \hfill$\blacksquare$
 
\begin{remark}[scale disparity of the two metrics]
\label{rem:scale}
By construction, the entries of the position columns of $\bH$ are
values of $h_{00}$ or $h_{01}$ and reach $1$
(Lemma~\ref{lem:basis}(c)), whereas all entries of the velocity columns
are values of $h_{10}$ or $h_{11}$ and are bounded in magnitude by
$4/27$. An identity-metric loss on $\bth$ therefore penalizes velocity
errors at least $27/4\approx 6.75$ times more strongly than their
worst-case pointwise influence on the decoded trajectory, which the
metric $\bH^\top\bH$ in Eq.~(21) corrects automatically. The same fact
explains the small values of $\sigma_{\min}(\bH)$ in
Table~\ref{tab:spectrum}; see Remark~\ref{rem:amp}.
\end{remark}
 
\subsection{Proof of Proposition 3.2 (expressivity and bandwidth)}
 
\noindent\textbf{Part (i).}
Fix a segment $k$ and write $s=s_k$, $F(u):=f(i_k+u\,s)$ for
$u\in[0,1]$, so that $F\in C^2([0,1])$ and
$\|F''\|_\infty=s^2\,\|f''\|_{\infty,[i_k,i_{k+1}]}$. The target
encoder of Eq.~(10) yields
\begin{gather*}
p_s=F(0),\qquad p_e=F(1),\\
v_s=s\big(f(i_k{+}1)-f(i_k)\big)=s\big(F(\tfrac1s)-F(0)\big),\\
v_e=s\big(F(1)-F(1{-}\tfrac1s)\big),
\end{gather*}
and the scaffold on this segment is
$g(u)=h_{00}(u)F(0)+h_{10}(u)v_s+h_{01}(u)F(1)+h_{11}(u)v_e$.
 
\emph{Knot exactness.} By Lemma~\ref{lem:basis}(a), $g(0)=F(0)=f(i_k)$
and $g(1)=F(1)=f(i_{k+1})$. Since the two copies of an interior knot in
$\bth_{\mathrm{gt}}$ read the same ground-truth frame (Alg.~1), the
piecewise scaffold is continuous and exact at every knot.
 
\emph{Uniform error.} By the mean value theorem, $v_s=F'(\xi_s)$ with
$\xi_s\in(0,\tfrac1s)$, $v_e=F'(\xi_e)$ with
$\xi_e\in(1{-}\tfrac1s,1)$, and the chord slope satisfies
$\Delta F:=F(1)-F(0)=F'(\eta)$ with $\eta\in(0,1)$. Let
$\ell(u)=(1-u)F(0)+uF(1)$ be the linear interpolant. Using
Lemma~\ref{lem:basis}(b),
\begin{align*}
g(u)-\ell(u)
&=\big[h_{00}(u)-(1-u)\big]F(0)+\big[h_{01}(u)-u\big]F(1)\\
&\qquad+h_{10}(u)\,v_s+h_{11}(u)\,v_e\\
&=\big[h_{10}(u)+h_{11}(u)\big]\big(F(0)-F(1)\big)\\
&\qquad+h_{10}(u)\,v_s+h_{11}(u)\,v_e\\
&=h_{10}(u)\big(v_s-\Delta F\big)+h_{11}(u)\big(v_e-\Delta F\big).
\end{align*}
Applying the mean value theorem to $F'$ between $\xi_s$ (resp.\
$\xi_e$) and $\eta$ gives $|v_s-\Delta F|\le\|F''\|_\infty$ and
$|v_e-\Delta F|\le\|F''\|_\infty$, hence by Lemma~\ref{lem:basis}(c)
\begin{equation*}
\big|g(u)-\ell(u)\big|
\le\big(h_{10}(u)+|h_{11}(u)|\big)\,\|F''\|_\infty
\le\tfrac14\,\|F''\|_\infty .
\end{equation*}
Combining with the classical linear-interpolation remainder
$|F(u)-\ell(u)|=\tfrac12\,u(1-u)\,|F''(\zeta_u)|\le\tfrac18\|F''\|_\infty$,
the triangle inequality yields, for every $u\in[0,1]$,
\begin{equation*}
\begin{aligned}
\big|F(u)-g(u)\big|
&\le\Big(\tfrac18+\tfrac14\Big)\|F''\|_\infty
\\=&\tfrac38\,s^2\,\|f''\|_\infty
\le\tfrac38\,s_{\max}^2\,\|f''\|_\infty ,
\end{aligned}
\end{equation*}
where time is measured in control steps. Taking the maximum over
segments proves part~(i). The two contributions are the classical
linear-interpolation remainder ($1/8$) and the error incurred by
estimating the endpoint tangents with one-sided finite differences
($1/4$); the latter is the price of the encoder of Eq.~(10) and would
vanish for exact endpoint derivatives, which is why the rate is
$O(s^2\|f''\|_\infty)$ rather than the $O(s^4\|f''''\|_\infty)$ of
Hermite interpolation with exact tangents. No rank assumption is used,
and the argument is valid for every $s_k\ge 1$ (for $s_k=1$ both
one-sided differences coincide with the chord). \hfill$\blacksquare$
 
\medskip
\noindent\textbf{Part (ii).}
On segment $k$ the scaffold is a cubic polynomial in $u$, so its second
derivative in $u$ is affine and its third derivative is constant; by
Lemma~\ref{lem:basis}(e),
\begin{equation*}
\frac{d^3 g}{du^3}
=12\big(p^{(k)}_s-p^{(k)}_e\big)+6\big(v^{(k)}_s+v^{(k)}_e\big).
\end{equation*}
With the time change $t=i_k+u\,s_k$ (hence $d/dt=s_k^{-1}\,d/du$), the
second derivative of $\ba^{c}$ is piecewise linear in $t$ and the third
is piecewise constant with at most $K$ distinct values,
\begin{equation*}
\frac{d^3 a^{(k)}}{dt^3}
=\frac{12\big(p^{(k)}_s-p^{(k)}_e\big)+6\big(v^{(k)}_s+v^{(k)}_e\big)}{s_k^{3}} ,
\end{equation*}
so with the per-step endpoint velocities $\nu^{(k)}=v^{(k)}/s_k$ the
triangle inequality gives exactly the bound of Eq.~(22). Under the
rounded uniform grid, $s_k\approx (T-1)/K$, so the admissible ceiling
scales as $K^{3}$ while, by part~(i), the scaffold approximation error
of the target decays as $K^{-2}$; the two bounds pull in opposite
directions, predicting an interior optimum in $K$.
\hfill$\blacksquare$
 
\begin{remark}[physical derivatives in a displacement action space]
\label{rem:phys}
The action channels are per-step displacements, so the executed
position is the running sum of $\ba^{c}$ and the jerk measured in
Sec.~4.5 is the \emph{second} derivative of $\ba^{c}$. Part~(ii)
characterizes derivatives of $\ba^{c}$ itself; translating, a scaffold
in the span of $\mathcal D_H$ yields a piecewise \emph{quartic}
position trajectory whose jerk is piecewise \emph{linear} with at most
$K$ pieces. Explicitly, by Lemma~\ref{lem:basis}(d), on segment $k$,
\begin{equation*}
\frac{d^2 a^{(k)}}{dt^2}\bigg|_{u}
=\frac{c^{(k)}_1\,u-c^{(k)}_0}{s_k^{2}},
\quad
\begin{aligned}
c^{(k)}_1&=12\Delta p^{(k)}+6\big(v^{(k)}_s+v^{(k)}_e\big),\\
c^{(k)}_0&=6\Delta p^{(k)}+4v^{(k)}_s+2v^{(k)}_e,
\end{aligned}
\end{equation*}
with $\Delta p^{(k)}=p^{(k)}_s-p^{(k)}_e$. This is affine in $u$ and
therefore attains its extrema at the endpoints, giving
\begin{equation*}
\sup_t\Big|\frac{d^2a^{(k)}}{dt^2}\Big|
\le\frac{6\big|\Delta p^{(k)}\big|}{s_k^{2}}
+\frac{4\big(|\nu^{(k)}_s|+|\nu^{(k)}_e|\big)}{s_k}.
\end{equation*}
The measured jerk ceiling thus scales as $K^{2}$ while the
approximation error of part~(i) decays as $K^{-2}$, so the trade-off
described above is unaffected. The piecewise-linear structure with only
$K=2$ pieces is what structurally precludes the isolated transient
spikes visible in Fig.~6.
\end{remark}
 
\subsection{Proof of Proposition 3.3 (seam control)}
 
\noindent\emph{Step 1 (handover states as linear functionals).}
Let $W\le T$ denote the execution window, so that the last executed
timestep of a chunk is $W{-}1$, lying in segment $k^\ast=k(W{-}1)$ with
local coordinate $u^\ast=(W{-}1-i_{k^\ast})/s_{k^\ast}$. On the
\emph{incoming} side of a handover, the first action of the new chunk
is its timestep $0$, i.e., $u=0$ of segment $0$; by
Lemma~\ref{lem:basis}(a) the scaffold position and velocity there are
exactly the boundary coordinates $\hat p^{\,\prime(0)}_s$ and
$\hat v^{\,\prime(0)}_s$ of $\hat\bth'$, independently of $W$. On the
\emph{outgoing} side, the scaffold position and velocity at timestep
$W{-}1$ are the fixed linear functionals $\br_p^\top\hat\bth$ and
$\br_v^\top\hat\bth$, where $\br_p$ is row $W{-}1$ of $\bH$ (supported
on block $k^\ast$ with entries $h_{\bullet}(u^\ast)$) and $\br_v$
tabulates $h'_{\bullet}(u^\ast)/s_{k^\ast}$ on the same block. If
$W{-}1$ coincides with a segment knot, then $u^\ast\in\{0,1\}$ and, by
Lemma~\ref{lem:basis}(a), $\br_p$ and $\br_v$ are coordinate selectors
(up to the fixed scaling $1/s_{k^\ast}$ for velocity): the outgoing
state reduces to a boundary-coordinate pair of $\hat\bth$ as well, in
particular to $(\hat p^{(K-1)}_e,\hat v^{(K-1)}_e)$ under full-chunk
execution ($W=T$).
 
\medskip
\noindent\emph{Step 2 (position: a trajectory functional).}
Because $\br_p$ is row $W{-}1$ of $\bH$, the outgoing scaffold position
is itself a trajectory component,
$\br_p^\top\hat\bth=(\bH\hat\bth)_{W-1}$, and by
Lemma~\ref{lem:basis}(a) the incoming one is
$\hat p^{\,\prime(0)}_s=(\bH\hat\bth')_0$. Inserting the two
demonstration frames and applying the triangle inequality,
\begin{equation}
\big|\br_p^\top\hat\bth-\hat p^{\,\prime(0)}_s\big|
\le\big\|\bH\hat\bth-\ba\big\|_\infty
+\Delta^{\mathrm{demo}}_{p}
+\big\|\bH\hat\bth'-\ba'\big\|_\infty ,
\label{eq:seampos}
\end{equation}
where $\Delta^{\mathrm{demo}}_{p}=\big|\ba[W{-}1]-\ba'[0]\big|$ is the
literal across-seam step of the demonstration.
Equation~\eqref{eq:seampos} involves neither $\sigma_{\min}(\bH)$ nor
$\|\br_p\|_2$, and holds for every window $W$ without requiring knot
alignment: the position component of a seam is controlled directly by
the trajectory-space residuals that Eqs.~(12), (17) and~(19) minimize.
 
\medskip
\noindent\emph{Step 3 (velocity: a coordinate functional).}
The row $\br_v$ tabulates $h'_\bullet$, which is not a row of $\bH$, so
the velocity component must pass through the coordinates. By the
identity of Proposition~3.1, for any target $\ba$,
\begin{equation*}
\big\|\bH\hat\bth-\ba\big\|_2^2
\ge\big\|\bH\big(\hat\bth-\bth_{\mathrm{LS}}(\ba)\big)\big\|_2^2
\ge\sigma_{\min}(\bH)^2\big\|\hat\bth-\bth_{\mathrm{LS}}(\ba)\big\|_2^2 ,
\end{equation*}
using Lemma~\ref{lem:rank}, hence
$\|\hat\bth-\bth_{\mathrm{LS}}(\ba)\|_2
\le\|\bH\hat\bth-\ba\|_2/\sigma_{\min}(\bH)$.
Writing $j'$ for the index of $v^{(0)}_s$ within $\bth'$ and using
$|\br^\top\mathbf x|\le\|\br\|_2\|\mathbf x\|_2$ together with
$|x_{j'}|\le\|\mathbf x\|_2$,
\begin{equation}
\big|\br_v^\top\hat\bth-\hat v^{\,\prime(0)}_s\big|
\le\|\br_v\|_2\frac{\|\bH\hat\bth-\ba\|_2}{\sigma_{\min}(\bH)}
+\Delta^{\mathrm{demo}}_{v}
+\frac{\|\bH\hat\bth'-\ba'\|_2}{\sigma_{\min}(\bH)} ,
\label{eq:seamvel}
\end{equation}
with $\Delta^{\mathrm{demo}}_{v}$ the across-seam velocity increment of
the demonstration in least-squares coordinates. By
Lemma~\ref{lem:basis}(c) and (f), the functional row norms satisfy
\begin{equation*}
\|\br_p\|_2\le\sqrt{1+2\big(\tfrac{4}{27}\big)^{2}}<1.03 ,
\qquad
\|\br_v\|_2\le\frac{\sqrt{13/2}}{s_{k^\ast}}\approx\frac{2.55}{s_{k^\ast}} ,
\end{equation*}
so the velocity functional is contractive whenever $s_{k^\ast}\ge 3$.
Together, Eqs.~\eqref{eq:seampos} and~\eqref{eq:seamvel} bound the seam
discontinuity in position and velocity by the prediction errors of the
two adjacent chunks plus the across-seam increment intrinsic to the
demonstration, which proves the claim. \hfill$\blacksquare$
 
\begin{remark}[magnitude of the amplification factor]
\label{rem:amp}
The factor $1/\sigma_{\min}(\bH)$ appearing in Eq.~\eqref{eq:seamvel}
equals $53.0$ at $T=10,K=2$ and $9.5$ at $T=50,K=2$
(Table~\ref{tab:spectrum}). Its size has the same origin as
Remark~\ref{rem:scale}: the right singular direction attaining
$\sigma_{\min}$ places $97.6\%$ of its mass on velocity coordinates,
whose columns are bounded by $4/27$ and therefore influence the decoded
trajectory only weakly. The velocity bound is consequently loose, and
markedly looser at the simulation horizon than at the real-robot
horizon at which the seam measurements of Table~12 are taken. The
position bound of Eq.~\eqref{eq:seampos} is free of this factor.
\end{remark}
 
\begin{remark}[discrete variant]
\label{rem:disc}
For \hdh{} the deployed $\hat\bth$ is dequantized (Eq.~(13)); uniform
quantization with bin width $\Delta_j$ per coordinate perturbs the
residual by at most an additive
$\|\bH\|_2\,\tfrac12\big(\sum_j\Delta_j^2\big)^{1/2}$, so
Eqs.~\eqref{eq:seampos} and~\eqref{eq:seamvel} hold with this extra
additive term.
\end{remark}
 
\begin{corollary}[knot-aligned and full-chunk execution]
\label{cor:knot}
When $W{-}1\in\{i_1,\dots,i_K\}$, both sides of every handover are
exactly boundary-coordinate pairs of the respective $\hat\bth$, and
Eq.~\eqref{eq:seamvel} holds with $\|\br_v\|_2=1/s_{k^\ast}$. In the
protocols of the main paper, the LIBERO window $W=5$ satisfies
$W{-}1=4=i_1$ on the rounded uniform grid ($T=10$, $K=2$), so the
simulated seams are knot-aligned, whereas the real-robot window $W=20$
places $W{-}1=19$ strictly inside the first segment ($T=50$, $K=2$) and
is covered by the general statement. Equation~\eqref{eq:seampos} is
insensitive to this distinction.
\end{corollary}
 
\begin{remark}[encoder targets instead of least-squares targets]
\label{rem:enc}
In Eq.~\eqref{eq:seamvel}, replacing $\bth_{\mathrm{LS}}$ by the
encoder targets $\bth_{\mathrm{gt}}$ of Alg.~1 makes the demonstrator
increment literal, at the price of one extra additive term
$\|\bH\bth_{\mathrm{gt}}-\ba\|_2/\sigma_{\min}(\bH)$ per chunk, bounded
via Proposition~3.2(i) by
$\sqrt{T}\,\tfrac38\,s_{\max}^2\|f''\|_\infty/\sigma_{\min}(\bH)$.
Equation~\eqref{eq:seampos} already uses the literal increment.
\end{remark}
 
\begin{remark}[identifiability ceiling and the $K$-ablation]
\label{rem:ceiling}
At $T=10$, Assumption~\ref{asm:ident} forces $K\le 2$: already at $K=3$
the operator has more columns than rows ($4K=12>T$), and two of its
three blocks satisfy $m_k=3<4$, so $\bH$ is column-rank deficient and
$\sigma_{\min}(\bH)=0$ (Table~\ref{tab:spectrum}). The boundary
coordinates are then unidentifiable from the decoded trajectory and the
velocity guarantee of Eq.~\eqref{eq:seamvel} degenerates, although
training remains well-posed because the objectives of Eqs.~(17)
and~(19) are posed in trajectory space and Eq.~\eqref{eq:seampos} is
unaffected. The empirically optimal configuration $K=2$ of Table~6
therefore coincides with the largest identifiable segment count at this
horizon, which we regard as a complementary explanation of the
degradation observed at $K=3,4$ alongside the bias--fidelity trade-off
of Proposition~3.2. At the real-robot horizon $T=50$ the constraint is
inactive for all $K\le 12$, so there $K=2$ reflects the trade-off
rather than an identifiability limit; we did not sweep $K$ at $T=50$,
and note this in Sec.~5.
\end{remark}

\end{document}